\documentclass[accepted]{uai2022} 
\usepackage[american]{babel}
\usepackage{mathtools}
\usepackage[utf8]{inputenc} 
\usepackage[T1]{fontenc}    

\usepackage{hyperref}       
\usepackage{url}            
\usepackage{booktabs}       
\usepackage{amsfonts}       
\usepackage{nicefrac}       
\usepackage{microtype}      
\usepackage{pgfplots}
\pgfplotsset{compat=1.12}
\definecolor{crimson}{RGB}{220, 20, 60}
\definecolor{indigo}{RGB}{75, 0, 30}

\usepackage{subfig}
\usepackage{amsmath}
\usepackage{amssymb}
\usepackage{graphicx}
\usepackage{bbm}

\usepackage{enumitem} 
\usepackage{amssymb,amsmath, amsthm}
\newtheorem{proposition}{Proposition}
\newtheorem{definition}{Definition}

\newcommand{\R}{\mathbb{R}}

\DeclareMathOperator*{\argmax}{argmax}
\DeclareMathOperator*{\argmin}{argmin}

\makeatletter
\newcommand{\printfnsymbol}[1]{%
  \textsuperscript{\@fnsymbol{#1}}%
}
\makeatother
\title{VQ-Flows: Vector Quantized Local Normalizing Flows}

\author[1]{Sahil Sidheekh\printfnsymbol{1}}
\author[2]{Chris B. Dock\thanks{Equal contribution}}
\author[1]{Tushar Jain}
\author[2]{Radu Balan}
\author[3]{Maneesh K. Singh\thanks{Work was performed while at Verisk Analytics.}}

\affil[1]{%
    Verisk Analytics
}
\affil[2]{%
    University of Maryland, College Park
}
\affil[3]{%
    Motive Technologies, Inc.
}

\begin{document}
\maketitle
\begin{abstract}
Normalizing flows provide an elegant approach to generative modeling that allows for efficient sampling and exact density evaluation of unknown data distributions. However, current techniques have significant limitations in their expressivity when the data distribution is supported on a low-dimensional manifold or has a non-trivial topology. We introduce a novel statistical framework for learning a mixture of \textit{local} normalizing flows as ``chart maps'' over the data manifold. Our framework augments the expressivity of recent approaches while preserving the signature property of normalizing flows, that they admit exact density evaluation. We learn a suitable atlas of charts for the data manifold via a vector quantized auto-encoder (VQ-AE) and the distributions over them using a conditional flow. We validate experimentally that our probabilistic framework enables existing approaches to better model data distributions over complex manifolds.
\end{abstract}

\section{Introduction}

Generative modeling is a machine learning paradigm that aims to learn data distributions and sample from it. If the data is drawn from a random variable $x\sim p(x)$, then one way to do this is to directly model $p(x)$ via a parameterized model ($\theta$) so that $p_\theta(x)\approx p(x)$. Such a model can then be used to generate new samples, which are expected to be statistically indistinguishable from the observed samples. Moreover, generative models that learn $p(x)$ are useful for data augmentation, outlier detection, domain transfer \cite{kobyzev2020normalizing,gen-domain-adapt}, and as priors for other downstream tasks \cite{pan2020dgp, flow-inverse-problems,density-modeling-image-prior}.

Among the most successful generative models are deep latent variable models, which assume that the latent factors of variation underlying the generative process of the data follow a simple distribution, such as a Gaussian or a uniform distribution. The non-linear function transforming this latent space to the data space (or vice-versa) is parameterized as a neural network and learned using gradient descent. Depending upon their formulation, there are three broad categories of deep latent variable models - GANs \cite{GANs}, VAEs \cite{VAE}, and normalizing flows. 

In this work, we focus on normalizing flows, a class of deep latent variable models introduced in \cite{tabak2013family} that support efficient sampling, exact density estimation, and inference \cite{rezende2015variational}. A normalizing flow maps the data space to a latent space through a series of diffeomorphisms (differentiable, bijective transformations with differentiable inverses). The data is assumed to follow an analytically computable distribution in the latent space, typically a Gaussian. Since the mapping is a diffeomorphism, the density in the data space can be obtained using the change of variables formula. To generate new samples using a flow, one can sample from the latent distribution and use the inverse transformation to map them to the data space. This makes normalizing flows powerful generative models that support exact density evaluation in contrast to GANs and VAEs.

Despite the advantages of normalizing flows over other generative models, their diffeomorphic requirement poses several restrictions. Firstly, a continuous bijective transformation with continuous inverse preserves the topology of its domain. Therefore, the data space is required to be topologically equivalent to the support of the latent distribution, typically to $D$ dimensional Euclidean space since the latent distribution is assumed to be a Gaussian. However, real data distributions typically differ from Euclidean space in many topological respects, such as the number of connected components, the presence of holes, etc. A normalizing flow would thus fail to model such data distribution accurately. Note, as an aside, that other generative models like GANs also suffer from these topological issues \cite{khayatkhoei2018disconnected}. 

A particularly troubling consequence of the continuous invertibility of flow transformations is that they are dimensionality preserving. However, according to the \textit{manifold hypothesis}, high dimensional real-world data living in $\mathcal{X}\simeq \R^D$ is often supported on a $d<<D$ manifold of the embedding space. To efficiently learn such distributions using flows, one needs to design expressive transformations that can map from a $d$ dimensional latent space to a the $D$ dimensional data space without making learning intractable. Recent work using stochastically invertible tall matrices \cite{cunningham2021change} and dimension raising conformal embeddings \cite{ross2021conformal} have paved the way in designing such transformations, however in both works expressivity is limited by the fact that the dimension changing operations are restricted to be linear (in \cite{cunningham2021change}) or made up of M\"{o}bius transformations (in \cite{ross2021conformal}).
\begin{figure}[t!]
    \centering
    \begin{tabular}{lll}
    \subfloat[Real Data]{
    \includegraphics[width = 0.317\linewidth]{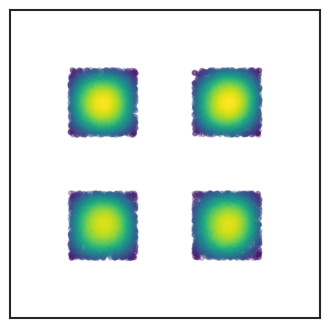}}
    \subfloat[Classic Flow]{
    \includegraphics[width =0.317\linewidth]{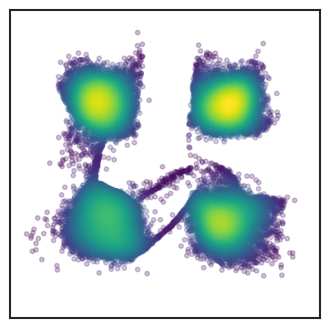}} 
    \subfloat[VQ-Flow ]{
    \includegraphics[width =0.317\linewidth]{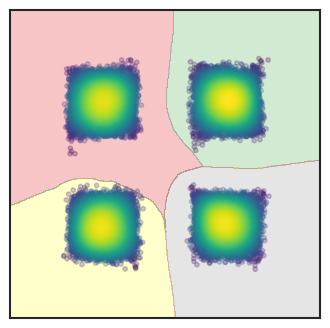}}
    \end{tabular}
    \caption{Augmentation of our framework (c) enables a classic flow (b) to better model the discontinuities in the data manifold through a learned atlas of charts(shaded region).}
\end{figure}

In this work, we propose to address the above limitations by parameterizing a family of normalizing flows to compose an atlas of charts over the data manifold. As the topology of the data manifold is expected to be ``locally'' equivalent to Euclidean space, a local normalizing flow should be able to model the local distribution over a chart region effectively. Further, by learning a mixture of flows over well-chosen charts, our approach compensates naturally for the limited expressiveness of existing flows. We summarize the main contributions of this work below:
\begin{itemize}
     \item We provide an understanding of the limited expressive power of existing flow-based models in modeling data distributions over complex topological spaces.
      \item We present a statistical framework for defining an expressive mixture of local normalizing flows that is flexible and generic enough to be used with existing approaches. We show that this framework allows for efficient sampling, inference of latent variables, and exact density evaluation while improving expressivity.
      \item We validate experimentally that the proposed approach improves flows for density estimation and sample generation, and is thus able to resolve many of the topological restrictions on expressivity imposed by using global diffeomorphisms.
\end{itemize}

\section{Background}

Given data $\{x_n\}_{n=1}^N \subset\mathcal{X}\simeq\R^D$ distributed according to an unknown distribution $p(x)$, a normalizing flow maps it through a diffeomorphism $f:\mathcal{X}\rightarrow\mathcal{Z}$ to a latent space $\mathcal{Z}\simeq \R^D$ such that $z=f(x)$ is simply distributed, for example $z\sim q(z)$ where $q=N(0,\mathbb{I})$. Recall that a diffeomorphism is a differentiable map that is bijective and whose inverse is also differentiable. Typically one denotes by $g$ the inverse of $f$ and parameterizes the normalizing flow as  $x=g_\theta(z) $,  where $\theta$ is the vector of learnable model parameters. The process of going from the latent space to data space is called {\em generation} or {\em sampling} and is accomplished by the function $g_\theta$, while the inverse procedure is termed {\em inference} and is accomplished by $f_\theta = g_{\theta}^{-1}$:
\begin{align}
    \begin{split}
    &f_\theta: \mathcal{X}\rightarrow \mathcal{Z}\\
    &\underbrace{x\mapsto f_\theta(x)}_{\mbox{Inference}}
    \end{split}
    \begin{split}
    &g_\theta: \mathcal{Z}\rightarrow \mathcal{X}\\
    &\underbrace{z\mapsto g_\theta(z)}_{\mbox{Sampling}}
    \end{split}
\end{align}
The approximation $p_\theta(x)$ to the true probability density $p(x)$ is then obtained from $q(z)$ through the change of variables formula as:
\begin{align}
    p_{\theta}(x) = q(f_\theta(x))|\det[J f_\theta(x)]|
\end{align}
As compositions of diffeomorphisms are also diffeomorphisms, one can design expressive flows by composing individual transformations that have simple to compute inverses and Jacobian determinants. Suppressing the vector of model parameters $\theta$, we will use the notation $f(x)=f^1\circ\cdots\circ f^L (x)$ where $f^1,\ldots,f^L$ are assumed to have easily computable Jacobian determinants and inverses. Define recursively $x^{l-1}=f^{l}(x^l)$, $1\leq l\leq L$, with $x^L=x$. Note that $x^l=f^{l+1}\circ\cdots\circ f^L(x)$ and $x^0=f(x)$. One can then write the log-likelihood as:
\begin{equation}
\begin{split}
    \log p(x) &= \log q(z) + \log \prod_{l=1}^{L}|\det[J f^l(x^l)]| \\ &= \log q(f(x)) +  \sum_{l=1}^{L}\log|\det[J f^l(x^l)]|
\end{split}
\end{equation}
A given layer $f^l$ of the normalizing flow will depend only on a subset $\theta_l$ of the parameters of $\theta:=(\theta_1,\ldots,\theta_L)$. Temporarily adding back in the $\theta$ dependence of $f_\theta$, maximum likelihood estimation of $\theta$ then yields the following optimization problem: 
\begin{equation}
\begin{split}
\theta^*&= \min_{\theta=(\theta_1,\ldots,\theta_L)}\frac{1}{N} \sum_{n=1}^N -\log p_{\theta}(x_n) \\ &= \min_{\theta=(\theta_1,\ldots,\theta_L)}\frac{1}{N} \sum_{n=1}^N \biggr{\{}-\log q(f_\theta(x_n))\\&-\sum_{l=1}^L\log|\det[J f_{\theta_l}^l(x_n^l)]|\biggr{\}}
\end{split}
\end{equation}
\section{Related Work}
Normalizing flows have come a long way since it was introduced in \cite{rezende2015variational,realnpv}, with much efforts focused on expanding their scalability and applicability. This has resulted in several different formulations \cite{huan2018neural-autoregresive-flows,jaini2019sspolyflow,behrmann2019iresnet,chen2018neuralode}, each with a multitude of proposed architectures \cite{maf,dinh205nice,kingma2016iaf,kingma2018glow,ho2019flow++,durkan2019nsf}, aimed at defining expressive yet analytically invertible flow transformations with efficiently computable jacobian determinants. However, as these approaches define invertible transformations in Euclidean space, they are dimensionality preserving and less suited for modeling distributions over lower dimensional manifolds \cite{dai2018diagnosing, behrman2021explodinginverse}. Subsequent works have tried to address this challenge by building injective flows \cite{brehmer2020mflow, cunningham2021change, cunningham2020normalizing, kothari2021trumpets,kumar2020regularized, caterini2021rectangular}. However, they trade off the benefits of dimensionality change to intractable density estimation or stochastic inverses. The work by \cite{ross2021conformal} overcomes the above limitations using conformal embeddings, but has limited expressive power, as we show in this work. One way to improve the expressivity of all the above approaches, and enable them to overcome topological constraints \cite{shakir2021flowreview}, is to relax their global diffeomorphic requirement by defining a \textit{mixture of flows}. Prior works in this direction have looked at infinite mixtures by defining flows in a lifted space \cite{dupont2019augmented} or by using continuous indexing \cite{cornish2020relaxing}. However, their added expressivity comes at the cost of tractable density computation, and one has to rely on variational approximations to train the model. A manifold geometric approach to normalizing flows is also taken in \cite{gemici2016normalizing} and \cite{mathieu2020riemannian}, however in contrast to this work these techniques assume the manifold and its Riemannian geometry are known. On similar lines with this work, \cite{dinh2019rad} proposes to use a finite mixture of flows through piecewise-invertible transformations over partitions of the data space by introducing both real and discrete valued latent variables in the flow. However, this formulation introduces discontinuities in the model density that leads to unstable training \cite{cornish2020relaxing}, necessitating the enforcement of boundary conditions through ad-hoc architectural changes. It is therefore limited in its generalizability to novel flow formulations. Our work, on the other hand, by decoupling the partition learning from the flow training, introduces a more generic and scalable framework that can aid existing flows to overcome topological constraints and learn complex data distributions efficiently.

\section{Methodology}

A traditional normalizing flow provides a global diffeomorphism between the latent space $\mathcal{Z}$ and the data space $\mathcal{X}\simeq\R^D$, and as such requires the latent space to have the same dimension as the data space. This can lead to numerical instability when the data is supported on a $d<D$ dimensional manifold $\mathcal{M}\subset \mathcal{X}$ because the learned transformation will tend to become ``less and less injective'' as it seeks to restrict its range to $\mathcal{M}$ \cite{dai2018diagnosing, behrman2021explodinginverse}. 

One way to overcome this challenge is to build transformations that map across dimensions while preserving invertibility on its image. Unfortunately, the natural approach of post-composing a $d$ dimensional bijective normalizing flow $g:\mathcal{Z}\rightarrow\mathcal{U}$ with a dimension-raising embedding $e:\mathcal{U}\rightarrow\mathcal{X}$ leads in general to an intractable likelihood since the determinant in the change of variables formula $p(x)=q(f(x))|Det[J_g J_e^T J_e J_g]|^{-\frac{1}{2}}$ no longer separates into a product of simpler determinants.  We will focus on the solution to this issue developed in \cite{ross2021conformal}, namely to post-compose the $d$ dimensional bijective normalizing flow $g:\mathcal{Z}\rightarrow \mathcal{U}$ with a dimension raising {\em conformal embedding} $c: \mathcal{U}\rightarrow\mathcal{X}$. An alternative solution developed in \cite{cunningham2021change} is to use a linear dimension raising embedding and invert it stochastically, but this approach relies on the dimension change operation being linear which is restrictive. The approach taken in \cite{ross2021conformal} hinges on the fact that for every $u\in\mathcal{U}$ the Jacobian $J_c(u)$ satisfies $J_c(u)^TJ_c(u)=\lambda(u)^2\mathbb{I}$ for $\lambda:\mathcal{U}\rightarrow\R$, thus
\begin{equation}
    \begin{split}
      \det[J_{c\circ g}^T J_{c\circ g}]^{\frac{1}{2}} &= \det[J_{g}^T J_c^T J_c  J_g]^{\frac{1}{2}}\\&=|\lambda(u)| \det[J_g^TJ_g]^{\frac{1}{2}} \\ &= |\lambda(u)| |\det[J_g]|
    \end{split}
\end{equation}
This splitting keeps the likelihood computation tractable, but the requirement that $\mathcal{M}$ be the range of a conformal embedding is artificially restrictive. This issue is exacerbated by the necessity of parameterizing $c$. As noted in \cite{ross2021conformal} the easiest way to do so is to let $c=c_J\circ\cdots\circ c_1$ where each $c_j$ is either a trivially conformal zero padding operation or a dimension preserving conformal transformation. A dimension preserving conformal transformation $f:\R^d\rightarrow\R^d$ with $d>2$ is restricted by Liouville's theorem to be a M\"{o}bius transformation, of the form $f(x)=(A,a,b,\alpha,\epsilon)(x) = b + \alpha(A x - a)/||A x-a||^{\epsilon}$ where $A\in O(d)$ is an orthogonal matrix, $\alpha\in\R$, $a,b\in\R^d$, and $\epsilon$ is either $0$ or $2$. Though it might initially appear that the composition of many such operations would give increased expressive power, the group structure of the M\"{o}bius transformations prevents this. Indeed, if $p_s:\R^d\rightarrow\R^{d+s}$ is the zero padding operation, $m_1=(A_1,a_1,b_1,\alpha_1,\epsilon_1)$ is a $d$ dimensional M\"{o}bius transformation and $m_2=(A_2,a_2,b_2,\alpha_2,\epsilon_2)$ is a $d+s$ dimensional M\"{o}bius transformation then it is easily verified that for $x\in\R^d$
\begin{align}
    m_2\circ p_{s}\circ m_1 (x)= (m_2\cdot \tilde{m_1}) ( p_s(x))
\end{align}
Where $\tilde{m_1}$ is the $d+s$ dimensional M\"{o}bius transformation
\begin{align}
    \tilde{m_1}=(\begin{bmatrix} A_1 & 0\\ 0 & \mathbb{I}_{s\times s}\end{bmatrix},p_s(a_1),p_s(b_1),\alpha_1,\epsilon_1)
\end{align}
Thus, this parametrization yields $c$ as a M\"{o}bius transformation of $\R^D$ composed with $p_{D-d}$. Practically speaking, if $c$ is parameterized as above, the assumption that $\mathcal{M}$ is the image of a global conformal embedding severely limits expressiveness. The class of global conformal embeddings is not subject to Liouville's theorem and is far richer than the set of M\"{o}bius transformations, but it is hard to parameterize.

\subsection{Differential geometry of conformally flat manifolds}
A weaker and more natural assumption than $\mathcal{M}$ being the image of a conformal embedding is that $\mathcal{M}$ is {\em locally conformally flat}. Recall that if $f:(\mathcal{N},\eta_1)\rightarrow(\mathcal{M},\eta_2)$ is a map between differentiable manifolds $\mathcal{N}$ and $\mathcal{M}$ with metrics $\eta_1: \mathcal{N}\times T\mathcal{N}\times T\mathcal{N}$ and $\eta_2:\mathcal{M}\times T\mathcal{M}\times T\mathcal{M}$ respectively then the pullback $f^* \eta_2$ of the metric $\eta_2$ through $f$ is defined via:
\begin{align}\begin{split}
    &f^*\eta_2: \mathcal{N}\times T\mathcal{N}\times T\mathcal{N}\rightarrow\R\\
    &f^*\eta_2(y,v,w) = \eta_2(f(y), Df(y)(v), Df(y)(w))
\end{split}\end{align}
With this in mind a $d$ dimensional manifold $\mathcal{M}$ is called {\em locally conformally flat} if $\eta_1=\sum_{i=1}^d d y_i^2$ is the flat metric and for any $x\in\mathcal{M}$ there is a neighborhood $U\ni x$, an open set $O\subset\R^d$, a diffeomorphism $f: O\rightarrow U$, and a differentiable scalar function $\lambda:O\rightarrow \R$ such that $f^*\eta_2(y,\cdot,\cdot)=\lambda(y) \eta_1(\cdot,\cdot)$ for all $y\in O$ \cite{lee2006riemannian}.  An alternate definition replaces $\R^d$ with a flat manifold (defined as having an identically vanishing Riemannian curvature tensor), but this definition is equivalent to the above since any $d$ dimensional flat manifold is locally isometric to $\R^d$ (not globally isometric, for example tori are flat when equipped with appropriate coordinates) \cite{gallot1990riemannian}.  In our case the metric $\eta_2$ is assumed to be inherited from the Euclidean metric on $\mathcal{X}\simeq \R^D$.

The notion of local conformal flatness provides far more flexibility than its global counterpart. It is well known, for example, that every $2$ dimensional Riemannian manifold is locally conformally flat, but even the sphere $S^2(\R)$ is not globally conformally flat (by contrast an explicit local conformal equivalence of $S^d(\R)$ to $\R^d$ is given by stereographic projection from the north and south poles) \cite{gallot1990riemannian}. In general, criteria are known for a Riemannian manifold of dimension $d>2$ to be locally conformally flat: For $d=3$ a pseudo-Riemanian manifold is locally conformally flat if and only if the Cotton tensor vanishes everywhere, for $d\geq 4$ a pseudo-Riemannian manifold is locally conformally flat if and only if the Weyl tensor vanishes everywhere \cite{gallot1990riemannian}.  The question of which manifolds are globally conformally flat is more difficult, and in applied problems this requirement is artificially restrictive. 

\subsection{Local Normalizing Flows}
We propose to break up the data manifold $\mathcal{M}$ into an atlas of overlapping charts $V_1,\ldots,V_K$. 
\begin{definition}[See \cite{lee2006riemannian}]
An atlas of (smooth) charts for a $d$ dimensional manifold $\mathcal{M}$ is a collection of subsets of $\mathcal{M}$, $\{V_k\}_{k=1}^K$ and a collection of subsets $\{P_k\}_{k=1}^K$ of $\R^d$ such that $\bigcup_{k=1}^K V_k =\mathcal{M}$ and a collection of invertible maps $f_k: V_k\rightarrow P_k$ such that the ``transition maps''  $f_i\circ f_j^{-1}: f_j(V_i\cap V_j)\rightarrow f_i(V_i\cap V_j)$ are smooth.
\end{definition}
We will assume charts of the form $V_j=U_j\cap\mathcal{M}$ and $P_j=f_j(V_j)$ where $U_j$ are learned open subsets of $\mathcal{X}$ such that $\{x_n\}_{n=1}^N\subset \bigcup_{k=1}^K U_k$ and $f_1,\ldots, f_K$ are conformal normalizing flows. In a slight abuse of terminology we will also refer to $U_1,\ldots, U_k$ as charts. To handle dimensionality change, we assume that the manifold $\mathcal{M}$ is locally conformally flat and of dimension $d$, implying that for $V_j$ sufficiently small there exists $D_k\subset\mathcal{U}$ and a conformal dimension raising map $c_k:\mathcal{U}\rightarrow\mathcal{X}$ so that $V_k=c_k(D_k)=c_k\circ g_k^L\circ\cdots\circ g_k^1(P_k)$.

Because chart regions may in general overlap, we propose to choose between them probabilistically. Specifically, given $U_1, \ldots, U_K$ that cover the data manifold $\mathcal{M}$, we model $p(x)$ via a latent random variable $z$ that takes values in $\mathcal{Z}$ and a ``chart picking'' discrete random variable $k$ that takes values in $\{1,\ldots,K\}$. For $k=1,\ldots,K$ let $g_k: \mathcal{Z}\rightarrow U_k$ be a global immersion (a differentiable injection whose Jacobian is everywhere full rank) with left inverse $f_k:V_k\rightarrow \mathcal{Z}$ and range $V_k=g_k(P_k)=\mathcal{M}\cap U_k$.

\begin{proposition}
Let $(U_k)_{k=1}^K$, $(V_k)_{k=1}^K$, $(g_k)_{k=1}^K$, and $(f_k)_{k=1}^K$ be as above. Further, let $k$ be a discrete random variable taking values $1,\ldots, K$ and $z$ a continuous random variable taking values in $\mathcal{Z}$. Then if $x$ is a random variable supported on $\mathcal{M}$ such that
\begin{align}
    p(x, z, k) = \delta(x- g_k(z)) q(z) p_k
\end{align}
One has
\begin{enumerate}[label=(\roman*)]
    \item The joint distribution of $x$ and $k$ is given by:
    \begin{align}
        p(x,k) = p_k \mathbbm{1}_{V_k}(x) |\det[J f_k (x) J f_k(x)^T]|^{\frac{1}{2}} q(f_k(x))
    \end{align}
    \item The marginals $p(k)$ and $p(z)$ are given by $p_k$ and $q(z)$ respectively.
    \item The random variables $z$ and $k$ are independent.
    \item The conditional distributions $p(x|k)$ and $p(k|x)$ are given by:
    \begin{align}
        \label{xgivenkprop}
        p(x|k) = \mathbbm{1}_{V_k}(x) |\det[J f_k (x) J f_k(x)^T]|^{\frac{1}{2}} q(f_k(x))\\
        \label{kgivenxprop}
        p(k|x) = \frac{p_k \mathbbm{1}_{V_k}(x) |\det[J f_k (x)J f_k (x)^T]|^{\frac{1}{2}} q(f_k(x))}{\sum_{j : x\in V_j}p_j |\det[J f_j (x) J f_j (x)^T]|^{\frac{1}{2}} q(f_j(x))}
    \end{align}
    \item The density of interest, $p(x)$ is given by
    \begin{align}\label{likelihood}
        p(x) = \sum_{k : x\in V_k} p_k |\det[J f_k (x)J f_k(x)^T]|^{\frac{1}{2}} q(f_k(x))
    \end{align}
\end{enumerate}
\end{proposition}
\begin{proof}
Deferred to the appendix (see Section \ref{appendix}).
\end{proof}

Thus we assume the joint distribution of $x$, $z$, and $k$ to be $p(x,z,k) = \delta(x-g_k(z))q(z)p_k$ and apply the above proposition. We will use either $q=N(0,\mathbb{I})$ or $q=\frac{1}{vol(B_1(0))}\mathbbm{1}_{B_1(0)}$ as our latent distribution and let $p_k$ be the normalized probability with which $x$ occurs in $U_k$, that is:
\begin{align}\begin{split}
    p_k &:= \frac{p(x\in U_k)}{\sum_{j=1}^K p(x\in U_j)}
    =\frac{\int_{U_k} p(x) dx}{\sum_{j=1}^K \int_{U_j} p(x)dx}
\end{split}\end{align}

\begin{figure}[t]
    \centering
    \includegraphics[width=0.99\linewidth]{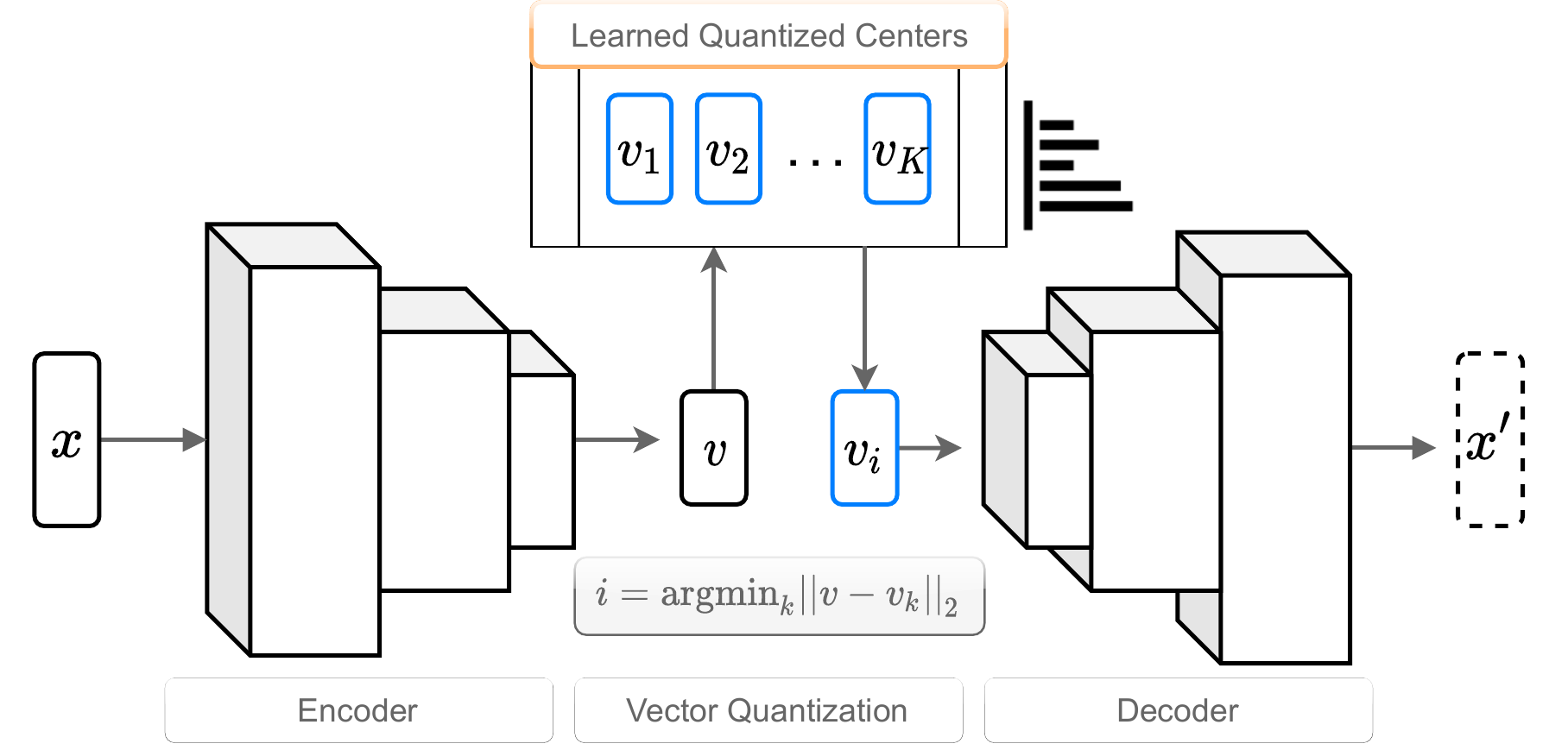}
    \caption{Learning quantized centers on the low dimensional data manifold using a vector quantized auto-encoder.}
\end{figure}

It remains to learn a ``good'' collection of charts $U_1,\ldots, U_K$, estimate $p_1,\ldots,p_K$, and then to parameterize $g_1,\ldots,g_K$ via normalizing flows $g_1^\theta,\ldots,g_K^\theta$ and obtain a maximum likelihood estimate for $\theta$ by optimizing $-\log p_\theta(x)$ (where $p_\theta(x)$ is as in \eqref{likelihood}).

\subsubsection{Learning the collection of charts $\mathbf{U_1,\ldots,U_K}$}
We learn the charts $U_1,\ldots,U_K$ via a vector-quantized auto encoder (VQ-AE)\cite{vqvae}, as it provides an effective and scalable mechanism to learn quantized centers on lower dimensional manifolds (also see \cite{mittal2022autosdf} for a recent application on high-dimensional data). The VQ-AE learns an encoder map $E: \mathcal{X}\rightarrow \mathcal{V}$, a decoder map $D: \mathcal{V}\rightarrow\mathcal{X}$, and a collection of ``encoded chart centers'' $Q=\{v_k\}_{k=1}^K\subset \mathcal{V}$ that minimize the reconstruction error $\mathcal{L}(D(\argmin_{v\in Q}||v-E(x)||_2), x)$. Once $D$, $E$, and $Q$ are learned we compute $d_k(x)=||E(x)-v_k||_2$ for $k=1,\ldots K$. With $d_1(x),\ldots d_k(x)$ in hand it remains to compute our charts. We would like the charts to overlap, but we also want them to be sparse in the sense that no individual $x$ has too many relevant charts. One possible choice is to fix $m\in\{1,\ldots,K\}$ and let $\tilde{d}_1\leq \cdots \leq \tilde{d}_K$ be the sorted permutation of $d_1,\ldots, d_K$ then define $U_k = \{x : ||E(x)-v_k||_2\leq\tilde{d}_m(x)\}$, so that every point $x$ has at least $m$ charts associated to it (those whose encoded chart centers are among the $m$ closest to $E(x)$). With this choice, a point $x$ will have exactly $m$ associated charts so long as the $m^{th}$ closest chart center is unique.
Another choice would be to fix $\epsilon > 0$ and let $U_k=\{ x : ||E(X)-v_k||_2 < (1+\epsilon) \tilde{d}_m(x)$\} (increasing $\epsilon$ enlarges each chart). For now we leave $m$ and $\epsilon$ as hyper-parameters, and in general denote $m(x)=|\{k : x\in U_k\}|$ (one always has $m(x)\geq m$). Note that checking if $x\in U_k$ amounts to computing $E(x)$ and $\tilde{d}_1(x),\ldots, \tilde{d}_K(x)$ and verifying that $||E(x)-v_k||_2 < (1+\epsilon)\tilde{d}_m(x)$. 

\subsubsection{Estimating $\mathbf{p_1,\ldots,p_K}$}
Once $U_1,\ldots,U_K$ are fixed note that if $r_k := p(x\in U_k)$,
\begin{align}
    r_k = \mathbb{E}_{x\sim p(x)}[\mathbbm{1}_{U_k}(x)]
\end{align}
The density $p(x)$ is unknown at this point, but we may estimate $r_k$ using the empirical distribution $\rho(x)=\frac{1}{N}\sum_{n=1}^N \delta(x-x_n)$ so that $r_k\approx\mathbb{E}_{x\sim \rho(x)}[\mathbbm{1}_{U_k}(x)]$. Practically speaking we thus perform a second pass over the training data and update $r_1,\ldots,r_K$ (initialized as zero) via $r_k^{(n)} = \frac{n-1}{n}r_k^{(n-1)} + \frac{1}{n}\mathbbm{1}_{U_k}(x_n)$, $1\leq n\leq N$, finally setting $r_k = r_k^{(N)}$ and $p_k = r_k / \sum_{j=1}^K r_j$.

\begin{figure}[t]
    \centering
    \includegraphics[width=0.99\linewidth]{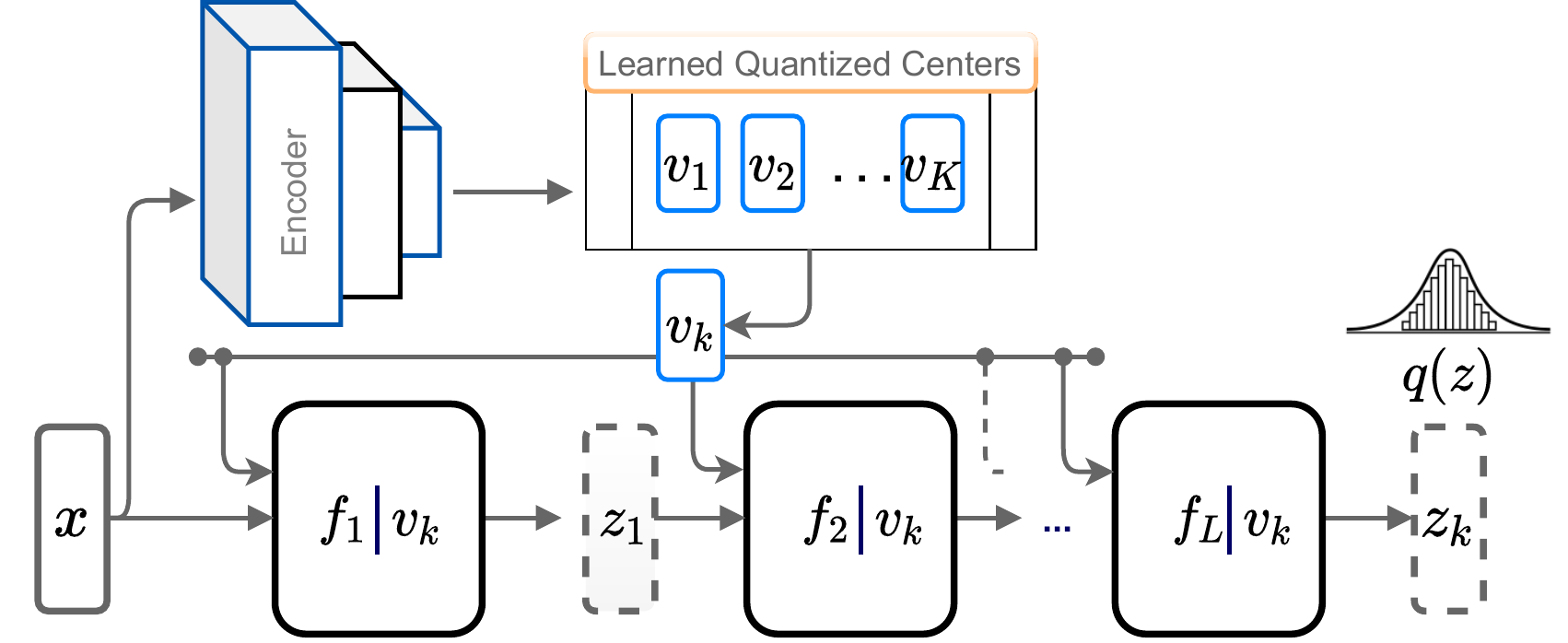}
    \caption{Learning the data distribution using a family of normalizing flows conditioned on the quantized centers.}
\end{figure}

\subsubsection{Learning the local transformations $\mathbf{g_1,\ldots,g_K}$}
Once $U_1,\ldots,U_K$ and $p_1,\ldots, p_K$ are obtained we model $g_k: \mathcal{Z}\rightarrow U_k$ as an $L$ layered invertible conditional normalizing flow. Where dimensionality change is required, we post-compose it with a conformal dimension raising map so that $g_k = c_k\circ g_k^L\circ\cdots\circ g_k^1$. We write the left inverse of $g_k$ via $f_k = f_k^1\circ\cdots\circ f_k^L\circ c_k^\dagger$ where $f_k^l= (g_k^l)^{-1}$ and $c_k^\dagger$ denotes the left inverse of the conformal map $c$ obtained by removing the zero padding and inverting the various M\"{o}bius transformations composing $c_k$. In practice, we reduce the number of parameters of our model by restricting each $g_k^l$ (and $f_k^l$) to depend on $k$ only through the value of the encoded chart center $v_k$. With this parametrization of $f_1,\ldots,f_K$ in hand \eqref{xgivenkprop} becomes
\begin{align}\begin{split}
    p(x|k)&=\mathbbm{1}_{V_k}(x)q(f_k(x))|\lambda_k(c_k^\dagger(x))|^{-1}\\&\prod_{l=1}^L|\det[J f_k^l(f_k^{l+1}\circ\cdots\circ f_k^L(x))]| 
\end{split}\end{align}
where $\lambda_k(u)$ is defined via $(J c_k (u))^T(J c_k (u)) = \lambda_k(u)^2 \mathbb{I}$. 

As we'll see this approach allows for far higher expressive power than global conformal flows without sacrificing the ability to generate realistic samples, perform inference, or compute exact densities. Indeed we may rewrite \eqref{likelihood} via
\begin{align}
\begin{split}
    p(x)& =\sum_{k: x\in U_k} p(x|k)p(k)\\
    & = \mathbb{E}_{k\sim \tilde{p}_x(k)}[p(x|k)] \underbrace{\sum_{j:x\in U_j} p(j)}_{\mbox{piecewise constant}}
\end{split}
\end{align}
Where $\tilde{p}_x(k) = p(k | p(x|k)>0)=p(k)/\sum_{j: x\in U_j} p(j)$. Thus, during {\em training} of the conditional normalizing flow we may replace the expectation $\mathbb{E}_{k\sim\tilde{p}(k)}[p(x|k)]$ with the stochastic quantity $p(x|k), k\sim \tilde{p}(k)$, performing only a single gradient descent pass per data-point as opposed to $m(x)$ passes. If the exact likelihood is needed, however, it can be computed at the cost of evaluating the normalizing flow and its Jacobian $m(x)$ times:
\begin{equation}\label{eqn:full_likelihood}
    \begin{split}
    p(x) &= \sum_{k: x\in U_k} p(x|k) p(k) \\ &= \sum_{k: x\in U_k} p_k q(f_k(x))|\lambda_k(c_k^\dagger(x))|^{-1}\\&\prod_{l=1}^L|\det[J f_k^l(f_k^{l+1}\circ\cdots\circ f_k^L(x))]| 
\end{split}
\end{equation}

Since $z$ and $k$ are independent, one can perform the {\em sampling task} via first sampling $z\sim q(z)$ and $k\sim p(k)$ and then computing a single forward pass of the normalizing flow chosen by $k$ to obtain $x=g_k(z)$. 

The {\em inference task} is complicated slightly by the fact that $z$ is no longer wholly determined given $x$, but instead takes values $(f_k(x))_{k: x\in U_k}$ with corresponding probabilities $(p(k|x))_{k: x\in U_k}$. One could perform a stochastic inference via sampling $k\sim p(k|x)$ and computing $z=f_k(x)$ (this amounts to choosing among the relevant charts for $x$), however if deterministic inference is preferred then of course one may always compute the expected value of $z$ as $z=\mathbb{E}_{k\sim p(k|x)}[f_k(x)]=\sum_{k: x\in U_k} p(k|x) f_k(x)$ or the most probable value of $z$ as $z= f_s(x)$ where $s=\argmax_{k: x\in U_k} p(k|x)$.

\subsection{Hard-boundary or deterministic approximation}
A particularly simple special case of the above model is the case $m=1$ and $\epsilon=0$, in which only a single chart is associated to a given $x$. This case reduces our atlas of overlapping charts to a disjoint partition of the data manifold $\mathcal{M}$. In this case $U_k$ is exactly the subset of $\mathcal{X}$ for whom $E(x)$ is closest to the encoded chart center $v_k$, and thus with the exception of $x$ lying on the chart boundaries, the random variable $k$ can be treated as a deterministic function of the random variable $x$, namely $k(x) = \argmin_{k=1,\ldots,K}||E(x)-v_k||_2 = \sum_{k=1}^K k \mathbbm{1}_{U_k}(x)$. Sampling in the hard-boundary case is identical to sampling in the soft-boundary case: generate samples for $x$ by first sampling $z\sim q(z)$ and $k\sim p(k)$ and then computing $x=g_k(z)$. Inference in the hard-boundary case is unambiguous since 
\begin{align}
\begin{split}
    &\mathbb{E}_{k\sim p(k|x)}[f_k(x)]=f_s(x) \\
    &s =\argmax_{k=1,\ldots,K} p(k|x)=\argmin_{k=1,\ldots,K} || E(x)-v_k||_2
\end{split}
\end{align}
That is to say that one performs inference by first identifying which region $R_s$ contains $x$ and then computing $z=f_s(x)$. The most significant simplification in the hard-boundary case from a computational standpoint comes in computing the likelihood $p(x)$, since if $x\in U_k$ then 
\begin{align}
\begin{split}
        p(x)&=p(x,k)=p(x | k) p(k) \\&= p(k) q(f_k(x))|\lambda_k(c_k^\dagger(x))|^{-1}\\&\prod_{l=1}^L|\det[J f_k^l(f_k^{l+1}\circ\cdots\circ f_k^L(x))]| 
\end{split}
\end{align}
Thus only one normalizing flow needs to be evaluated to compute the exact likelihood $p(x)$ (as opposed to $m(x)$ of them) and the normalizing flows may be trained using the exact likelihood as opposed to an unbiased estimator for it.

\section{Experiments}

\begin{table*}[t]
\centering
\begin{tabular}{@{}ccccccc@{}}
\toprule
\textbf{Model} & Spherical    & Helix         & Lissajous    & Twisted-Eight & Knotted          & Interlocked-Circles    \\ \midrule
Real NVP       & 3.15  $\pm$  0.07 & -3.37  $\pm$  0.16 & 2.42  $\pm$  0.07 & 0.94  $\pm$  0.15  & -2.17  $\pm$  0.14    & 0.95  $\pm$  0.13              \\
VQ-RealNVP     & 3.55  $\pm$  0.04 & -1.66  $\pm$  0.08 & 3.04  $\pm$  0.15 & 2.29  $\pm$  0.14  & 0.39  $\pm$  0.18     & 2.42  $\pm$  0.25           \\ \midrule
MAF            & 4.38  $\pm$  0.10 & -2.90  $\pm$  0.02 & 2.50  $\pm$  0.12 & 1.34  $\pm$  0.22  & -1.02  $\pm$  0.14    & 1.07  $\pm$  0.07           \\
VQ-MAF         & 4.43  $\pm$  0.14 & -0.49  $\pm$  0.03 & 3.48  $\pm$  0.16 & 2.01  $\pm$  0.10  & 0.62  $\pm$  0.16     & 2.29  $\pm$  0.18           \\ \midrule
CEF            & 0.91  $\pm$  0.07 & -3.71  $\pm$  0.09 & 0.42  $\pm$  0.15 & -0.38  $\pm$  0.21 & -2.48  $\pm$  0.26    & -0.72  $\pm$  0.11           \\
VQ-CEF         & 0.98  $\pm$  0.11 & -2.90  $\pm$  0.17 & 1.65  $\pm$  0.14 & -0.32  $\pm$  0.19 & -1.93  $\pm$  0.17    & 1.24  $\pm$  0.15           \\ \bottomrule
\end{tabular}  
\caption{ Quantitative evaluation of \textbf{Density Estimation} in terms of the test log-likelihood in nats (higher the better) on the 3D datasets. The values are averaged across 5 independent trials, $\pm$ represents the 95\% confidence interval.}
\label{density-est-quantitative}
\end{table*}

\begin{table*}[t!]
\centering
\begin{tabular}{@{}ccccccc@{}}
\toprule
\textbf{Model} & Spherical    & Helix         & Lissajous    & Twisted-Eight & Knotted          & Interlocked-Circles    \\ \midrule
Real NVP       & 0.50 $\pm$ 0.07 & -57.46 $\pm$ 2.11 & 0.18 $\pm$ 0.14 & -2.72 $\pm$ 0.90  & -8.65 $\pm$ 0.87    & -2.18 $\pm$ 0.37              \\
VQ-RealNVP     & 0.99 $\pm$ 0.14 & -3.85 $\pm$ 0.98  & 0.59 $\pm$ 0.08 & 0.18 $\pm$ 0.17  &  -1.44 $\pm$ 0.37    &  -0.11 $\pm$ 0.12          \\ \midrule
MAF            & 0.65 $\pm$ 0.26 & -92.83 $\pm$ 5.69 & 0.12 $\pm$ 0.16 & -2.77 $\pm$ 0.81  & -7.04 $\pm$ 0.49    & -2.49 $\pm$ 0.14           \\
VQ-MAF         & 1.01 $\pm$ 0.07 & -4.62 $\pm$ 0.37  & 0.59 $\pm$ 0.07 & -0.32$\pm$ 0.13  & -2.44 $\pm$ 0.11    & -0.15 $\pm$ 0.08            \\ \midrule
CEF            & -1.17 $\pm$ 0.06 & -29.90 $\pm$ 2.12 & 0.38 $\pm$ 0.14 & -4.03 $\pm$ 0.38 & -19.40 $\pm$ 1.80   & -3.42 $\pm$ 0.49           \\
VQ-CEF         & 0.80 $\pm$ 3.42  & -20.75 $\pm$ 2.22 & 0.49 $\pm$ 0.03 & -3.51 $\pm$ 0.73 & -14.44 $\pm$ 1.57   & -3.23 $\pm$ 0.19           \\ \bottomrule
\end{tabular}
\caption{ Quantitative evaluation of \textbf{Sample Generation} in terms of the log-likelihood of generated samples in nats (higher the better) on the 3D datasets. The values are averaged across 5 independent trials, $\pm$ represents the 95\% confidence interval.}
\label{sample-gen-quantitative}
\end{table*}

\begin{figure*}[t]
    \centering
    \begin{tabular}{cc}
    \subfloat[Spherical]{\includegraphics[width = 0.16\linewidth]{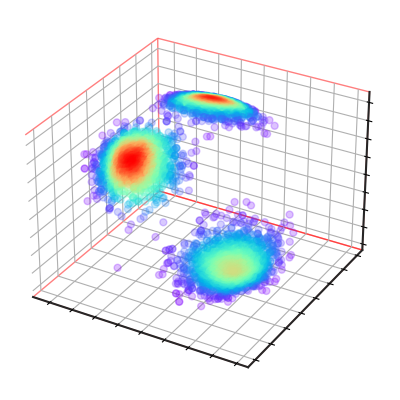}}
    \subfloat[Helix]{\includegraphics[width = 0.16\linewidth]{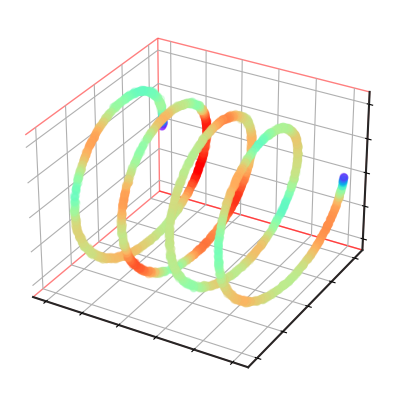}}
    \subfloat[Lissajous]{\includegraphics[width = 0.16\linewidth]{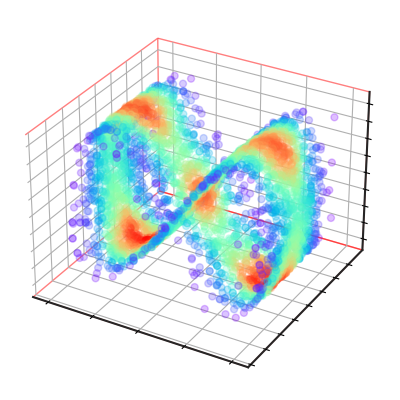}}
    \subfloat[Twisted-Eight]{\includegraphics[width = 0.16\linewidth]{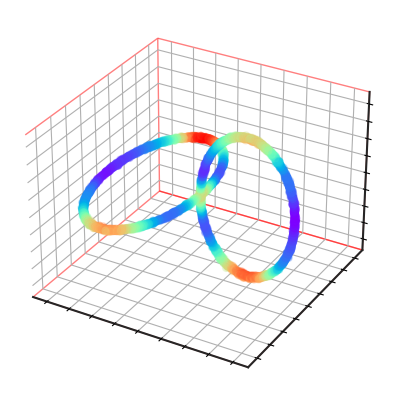}}
    \subfloat[Knotted]{\includegraphics[width = 0.16\linewidth]{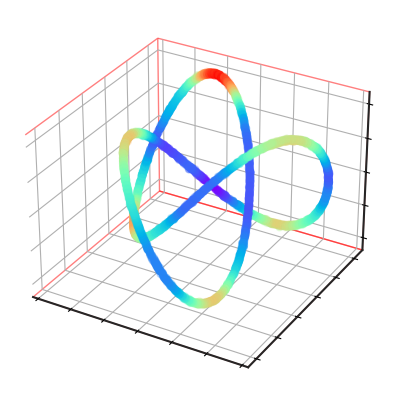}}
    \subfloat[InterlockedCircles]{\includegraphics[width = 0.16\linewidth]{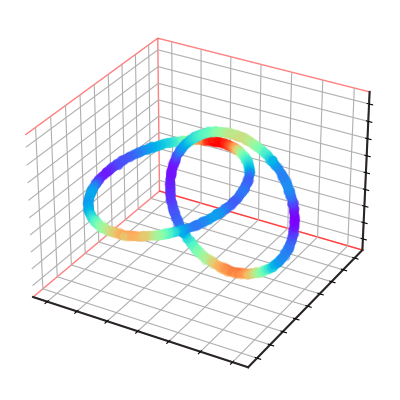}}
    \end{tabular}
    \includegraphics[width = 0.155\linewidth]{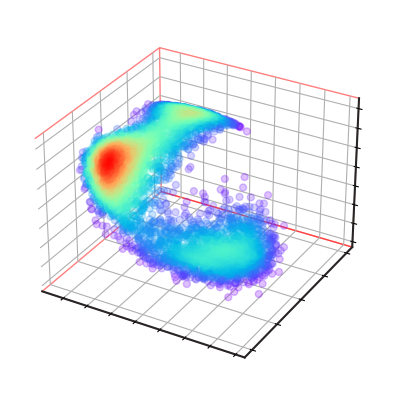}
    \includegraphics[width = 0.155\linewidth]{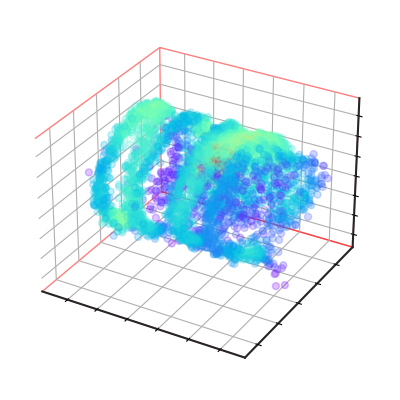}
    \includegraphics[width = 0.155\linewidth]{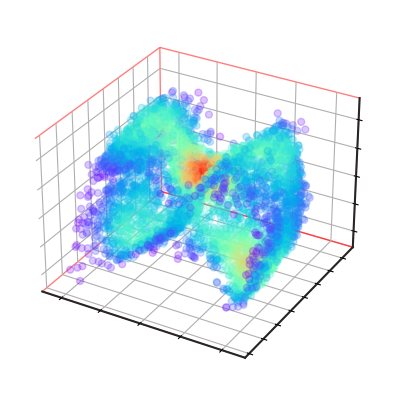}
    \includegraphics[width = 0.155\linewidth]{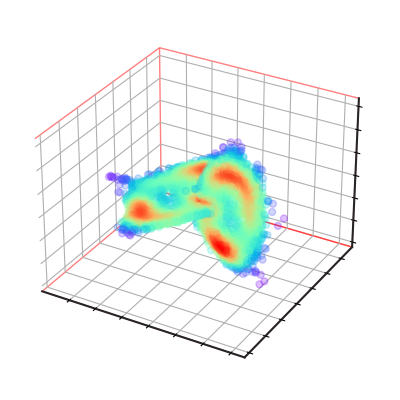}
    \includegraphics[width = 0.155\linewidth]{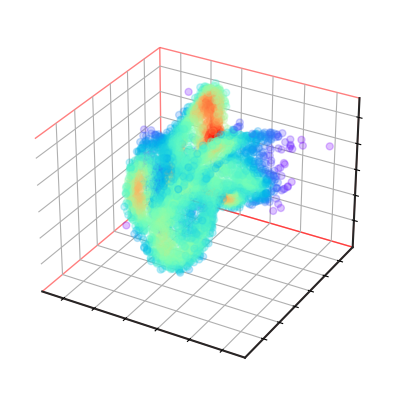} 
     \includegraphics[width = 0.155\linewidth]{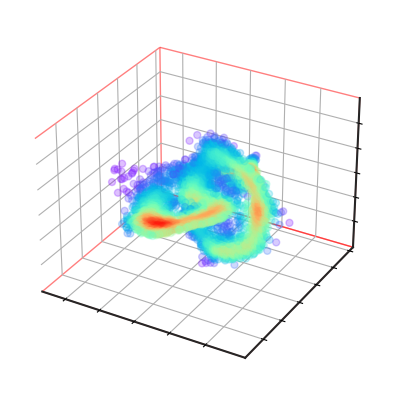} \\ 
    
    \includegraphics[width = 0.155\linewidth]{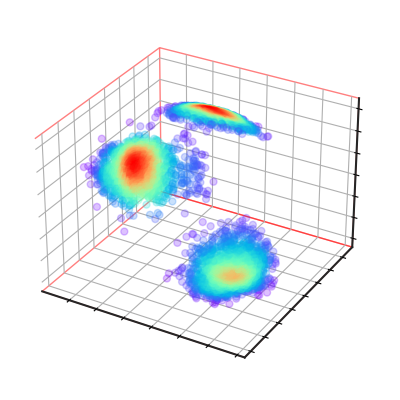}
    \includegraphics[width = 0.155\linewidth]{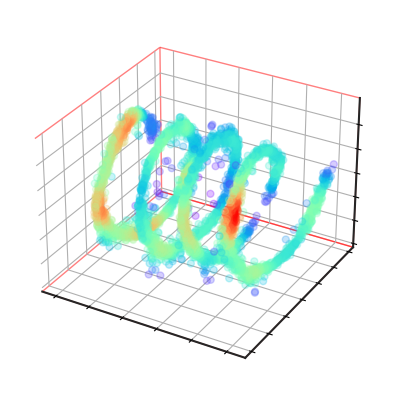}
    \includegraphics[width = 0.155\linewidth]{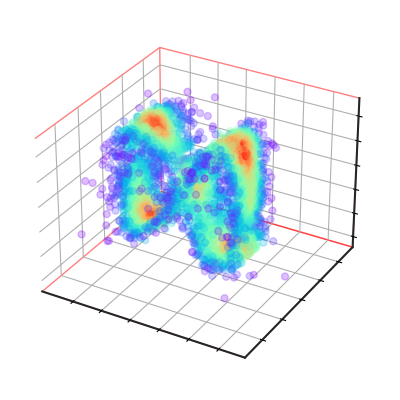}
    \includegraphics[width = 0.155\linewidth]{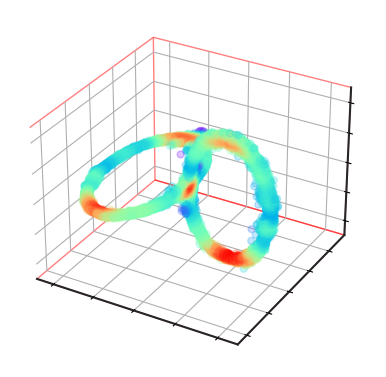}
    \includegraphics[width = 0.155\linewidth]{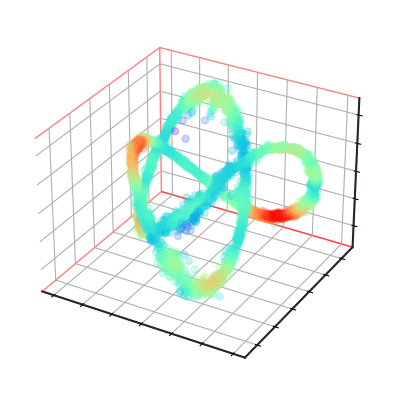}
    \includegraphics[width = 0.155\linewidth]{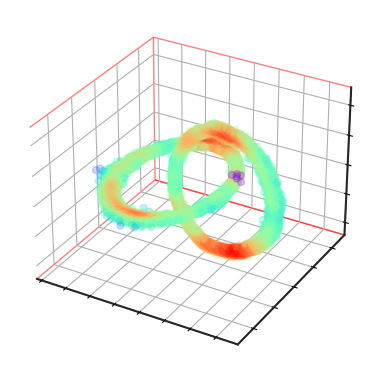}
    \caption{Qualitative visualization of the samples generated by a classical flow - RealNVP (Middle Row) and its VQ-counterpart (Bottom Row) trained on Toy 3D data distributions (Top Row).}
    \label{3D-data-qualitative-results}
\end{figure*}

To experimentally validate the efficacy of the proposed framework, we consider six 3-dimensional data distributions over manifolds of varying complexity as shown in Figure \ref{3D-data-qualitative-results}. Each dataset consists of $10,000$ datapoints, $5,000$ of which we use for training and $2,500$ each for validation and testing. We train three different normalizing flows - RealNVP \cite{realnpv}, Masked Autoregressive Flows (MAF) \cite{maf} and Conformal Embedding Flows (CEF) \cite{ross2021conformal} over these datasets with and without the augmentation of our framework. We refer to a base \textit{flow} augmented with the vector quantized conditioning as VQ-\textit{flow}. We define each model using $5$ flow transformations and train them for $100$ epochs using an Adam optimizer, early stopping if the validation performance does not improve over 10 epochs. For CEF, we use a 2-dimensional RealNVP as the base flow, which is then raised to the 3-dimensional space using the conformal embedding.  We parameterize the VQ-AE using feedforward neural networks and use a latent dimension of $2$ with $k=32$, to learn the partitioning of the data manifold. To define the conditional normalizing flow, we use the parameterization given in \cite{lu2020conditional}. We evaluate the models for density estimation and sample generation. We follow the same hyperparameters for a base flow and its VQ-counterpart without any tuning and report the performance averaged over $5$ independent trials. We defer further details on data generation, implementation as well as results on additional 3D data distributions to the supplementary material.

\subsection{Density Estimation}

The ability to compute exact likelihood is one of the critical features of a normalizing flow that makes it a potential tool in solving inverse problems. Improving the expressive power of flows can thus enhance their utility as priors by better modeling the data density. Thus, we first evaluate the proposed framework's ability to enhance the expressivity of flows to perform better density estimation. Table \ref{density-est-quantitative} compares the log-likelihood (in nats) achieved by different flow models with and without the VQ-augmentation on a held-out test set. A higher value indicates a better learned density. We observe that VQ-flows are able to achieve higher test log-likelihoods than their non-VQ-counterparts consistently across the considered data distributions. Thus, our framework enables better density estimation for normalizing flows over complex manifolds.

\subsection{Sample Generation}

A key desiderata of an expressive generative model is its ability to generate high fidelity samples from the data distribution. Figure \ref{3D-data-qualitative-results} visualizes the samples generated by a RealNVP flow trained on the 3D data distributions with and without the VQ augmentation. We observe that while the classical flow is able to generate samples from the data manifold, it also generates data points off the manifold, resulting in a poor fit to the real data distribution, as expected due to the requirements of being a global diffeomorphism. VQ-flows are seen to overcome these restrictions and generate samples better approximating the real data distribution. For a quantitative comparison, we evaluate the log-likelihood of the generated samples using a kernel density estimator fitted on the training data. We use a Gaussian kernel, with an optimal bandwidth obtained through cross-validation for each data distribution. We observe (Table \ref{sample-gen-quantitative}) that VQ-flows, owing to their ability to model the topology of the data manifold better, significantly outperform their non-VQ counterparts on sample generation.

\subsection{High Dimensional Data}
\begin{figure}[h!]
    \centering
     \begin{tabular}{cc}
    \subfloat[RealNVP]{\includegraphics[width=0.48\linewidth]{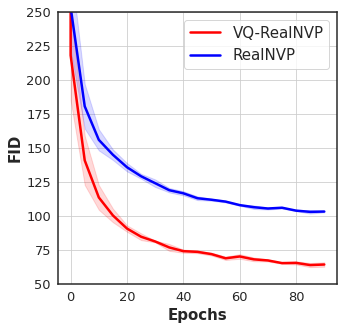}}
    \subfloat[MAF]{\includegraphics[width=0.49\linewidth]{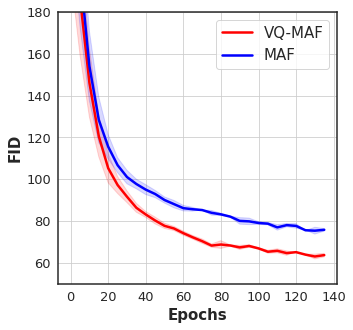}}
    \end{tabular}
    \caption{ FID scores (lower the better) across the training of (a) RealNVP and (b) MAF on the MNIST dataset. The shaded region represents the standard deviation over 3 trials.}
    \label{mnist-fid}
\end{figure}
To study the scalability of the proposed approach to higher dimensions, we consider the MNIST \cite{deng2012the} dataset comprising 60,000 grayscale images of handwritten digits, each of dimension 784 (28$\times$28). We train RealNVP and MAF with and without the VQ-augmentation and plot FID scores of the generated samples across their training iterations in Figure \ref{mnist-fid}. We observe that VQ-flows are able to achieve better performance (lower FID scores) faster than their non-VQ counterparts, hence validating the utility of the proposed approach in higher dimensions. An interesting observation here is that while MAF results on MNIST are much better than that of RealNVP, both VQ-MAF and VQ-RealNVP converge to the similar (low) FID scores. This early result seems to validate our hypothesis that the core difficulties (topology, dimensionality, etc), even on real datasets, can perhaps be better addressed by the proposed research direction than by improving backbone ‘single’ flow models. 

\subsection{Ablation Study} 
\begin{figure*}[h!]
    \centering
     \begin{tabular}{ccc}
    \subfloat[Sample Generation]{\includegraphics[width=0.28\linewidth]{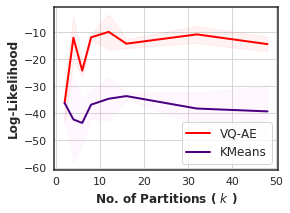}}
    \subfloat[Density Estimation]{\includegraphics[width=0.28\linewidth]{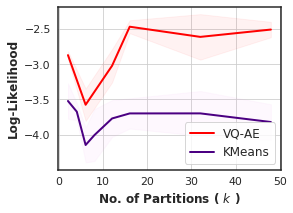}}
    \subfloat[Upon further training ]{\includegraphics[width=0.26\linewidth]{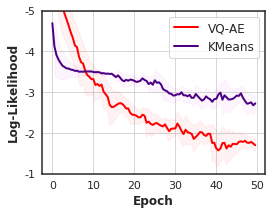}}
    \end{tabular}
    \caption{\textbf{Ablation Study} on the effect of the partitioning method and the number of partitions $k$ on sample generation (a) and density estimation (b). (c)-The learning trajectory of the flow for a fixed $k$(=32), in terms of validation log-likelihood. The shaded region represents the standard deviation over 3 independent trials.}
    \label{ablation}
\end{figure*}

Parameterizing the partitioning function using a VQ-AE is a design choice and the no. of partitions $k$ to consider over the data manifold is an important hyperparameter underlying the proposed framework. We conduct an ablation study to evaluate the sensitivity of our approach on $k$ and the partitioning method. We consider k-means clustering as an alternative design choice for the partitioning function. We train a RealNVP flow over the HELIX data distribution using k-means and VQ-AE, across increasing values of $k$. We plot the validation log-likelihood post training for $25$ epochs as a function of $k$ in Figure \ref{ablation}. We observe that VQ-AE results in better performance of the flow consistently across $k$, over k-means. Further, the choice of $k$ beyond a threshold does not have any significant effect on the model, hence it is sufficient to fix it to a large enough value.

\section{Future Work \& Conclusion}
Our framework is particularly well suited to high dimensional datasets (such as natural images) that obey the manifold hypothesis, an avenue we hope to explore in the sequel. One of the practical issues we encountered with our approach is that training  $g_k$  only on samples from $U_k$ does not always restrict the learned $p(x|k)$ to be supported only on $U_k$. In such cases, the sum over $k$ such that $x\in U_k$ in  \eqref{eqn:full_likelihood} yields an underestimate for $p(x)$, and the total sum  $k=1,\ldots, K$ must be used instead during testing. In the future, we hope to address this issue by explicitly discouraging the generation of samples outside $U_k$. 

To summarize, motivated by differential and conformal geometry, we have developed a novel probabilistic framework for ``local'' flows. We have demonstrated experimentally on toy data distributions with various topological features that this framework outperforms global flows - both dimension preserving (bijective flows) and dimension raising (embedding flows). Our framework is agnostic to the type of flow transformation employed and retains the key feature of normalizing flows: exact density evaluation. As such, we argue that using local flows as probabilistic chart maps over the data manifold is a natural way to overcome limited expressivity in the presence of dimension change or other topological impediments.

\section{Appendix}\label{appendix}
\begin{proof}[Proof of Proposition 1]
Proving $(i)$. 

We can compute the joint distribution over $x$ and $k$ - $p(x,k)$ as given below:
\begin{align} 
\label{xkjoint}
    &p(x,k)=\int_\mathcal{Z} p(x,z,k) dz = p_k\int_\mathcal{Z} \delta(x-g_k(z))q(z) dz \nonumber\\
    &= p_k \mathbbm{1}_{V_k}(x) \times \nonumber \\
    &\int_\mathcal{Z} \delta(z-f_k(x)) |\det[J g_k (z)^T J g_k(z)]|^{-\frac{1}{2}}q(z) dz \nonumber\\
    &= p_k \mathbbm{1}_{V_k}(x) |\det[J g_k (f_k(x))^T J g_k (f_k(x))]|^{-\frac{1}{2}}q(f_k(x)) \nonumber\\
    &= p_k \mathbbm{1}_{V_k}(x) |\det[J f_k (x) J f_k(x)^T]|^{\frac{1}{2}} q(f_k(x))
\end{align}

Proving $(ii)$. It is readily verified that $p(z) = q(z)$ and $p(k)=p_k$, in particular:
\begin{align}\begin{split}\label{zprior}
    p(z)&=\sum_{k=1}^K\int_\mathcal{X} p(x,z,k) dx\\
    &= \sum_{k=1}^K p_k \int_\mathcal{X} \delta(x-g_k(z))q(z) dx\\
    &= q(z) \sum_{k=1}^K p_k = q(z)
\end{split}\end{align}
and,
\begin{align}\begin{split}\label{kprior}
    p(k) &=\int_{X} p(x,k) dx\\
    &= p_k \int_{X}  \mathbbm{1}_{V_k}(x) |\det[J f_k (x)J f_k (x)^T]|^{\frac{1}{2}} q(f_k(x)) dx\\
    &= p_k \int_{V_k}  |\det[J f_k (x)J f_k (x)^T]|^{\frac{1}{2}} q(f_k(x)) dx\\
    &= p_k \int_{\mathcal{Z}} q(z) dz = p_k
\end{split}\end{align}
Proviing $(iii)$. Taken together, \eqref{zprior} and \eqref{kprior} yield that $z$ and $k$ are independent random variables since,
\begin{align}
    p(z,k)= \int_\mathcal{X} p(x,z,k) dx = p_k q(z)= p(k)p(z)
\end{align}
Proving $(iv)$. Dividing \eqref{xkjoint} by $p(k)=p_k$ we get that the distribution of $x$ conditioned on a particular chart is given by:
\begin{align}\label{xgivenk}
    p(x|k) = \mathbbm{1}_{V_k}(x) |\det[J f_k (x) J f_k(x)^T]|^{\frac{1}{2}} q(f_k(x))
\end{align}
In particular, $p(x|k)$ is zero unless $x\in U_k$. Meanwhile $p(k|x)$ is given by the Bayes' formula as:
\begin{align}\label{kgivenx}
\begin{split}
    p(k|x) &= \frac{p(x|k)p(k)}{\sum_{j=1}^K p(x|j)p(j)}\\
    &= \frac{p_k \mathbbm{1}_{V_k}(x) |\det[J f_k (x)J f_k (x)]|^{\frac{1}{2}} q(f_k(x))}{\sum_{j : x\in U_j}p_j |\det[J f_j (x) J f_j (x)^T]|^{\frac{1}{2}} q(f_j(x))}
\end{split}
\end{align}
Proving $(v)$. Note that the distribution $p(k|x)$ is thus also zero unless $x\in U_k$, a fact that will be employed during inference. Finally the density $p(x)$ is given by:
\begin{align}
\begin{split}
    p(x) &= \sum_{k=1}^M p(x|k)p(k)\\
    &= \sum_{k : x\in U_k} p_k |\det[J f_k (x)J f_k(x)^T]|^{\frac{1}{2}} q(f_k(x))
\end{split}
\end{align}
\end{proof}
\bibliographystyle{unsrt}
\bibliography{references}

\begin{thebibliography}{10}

\bibitem{kobyzev2020normalizing}
Ivan Kobyzev, Simon~JD Prince, and Marcus~A Brubaker.
\newblock Normalizing flows: An introduction and review of current methods.
\newblock {\em IEEE transactions on pattern analysis and machine intelligence},
  43(11):3964--3979, 2020.

\bibitem{gen-domain-adapt}
Guanglei Yang, Haifeng Xia, Mingli Ding, and Zhengming Ding.
\newblock Bi-directional generation for unsupervised domain adaptation.
\newblock In {\em The Thirty-Fourth {AAAI} Conference on Artificial
  Intelligence, {AAAI}}, pages 6615--6622, 2020.

\bibitem{pan2020dgp}
Xingang Pan, Xiaohang Zhan, Bo~Dai, Dahua Lin, Chen~Change Loy, and Ping Luo.
\newblock Exploiting deep generative prior for versatile image restoration and
  manipulation.
\newblock In {\em European Conference on Computer Vision (ECCV)}, 2020.

\bibitem{flow-inverse-problems}
Jay Whang, Erik Lindgren, and Alex Dimakis.
\newblock Composing normalizing flows for inverse problems.
\newblock In {\em Proceedings of the 38th International Conference on Machine
  Learning}, volume 139 of {\em Proceedings of Machine Learning Research},
  pages 11158--11169, 2021.

\bibitem{density-modeling-image-prior}
Johannes Ball{\'{e}}, Valero Laparra, and Eero~P. Simoncelli.
\newblock Density modeling of images using a generalized normalization
  transformation.
\newblock In {\em 4th International Conference on Learning Representations,
  {ICLR}}, 2016.

\bibitem{GANs}
Ian~J. Goodfellow, Jean Pouget{-}Abadie, Mehdi Mirza, Bing Xu, David
  Warde{-}Farley, Sherjil Ozair, Aaron~C. Courville, and Yoshua Bengio.
\newblock Generative adversarial nets.
\newblock In {\em Advances in Neural Information Processing Systems}, pages
  2672--2680, 2014.

\bibitem{VAE}
Diederik~P. Kingma and Max Welling.
\newblock Auto-encoding variational bayes.
\newblock In {\em 2nd International Conference on Learning Representations,
  {ICLR}}, 2014.

\bibitem{tabak2013family}
Esteban~G Tabak and Cristina~V Turner.
\newblock A family of nonparametric density estimation algorithms.
\newblock {\em Communications on Pure and Applied Mathematics}, 66(2):145--164,
  2013.

\bibitem{rezende2015variational}
Danilo~Jimenez Rezende and Shakir Mohamed.
\newblock Variational inference with normalizing flows.
\newblock In {\em Proceedings of the 32nd International Conference on Machine
  Learning, {ICML}}, volume~37 of {\em {JMLR} Workshop and Conference
  Proceedings}, pages 1530--1538, 2015.

\bibitem{khayatkhoei2018disconnected}
Mahyar Khayatkhoei, Maneesh Singh, and Ahmed Elgammal.
\newblock Disconnected manifold learning for generative adversarial networks.
\newblock In {\em Advances in Neural Information Processing Systems}, pages
  7354--7364, 2018.

\bibitem{cunningham2021change}
Edmond Cunningham and Madalina Fiterau.
\newblock A change of variables method for rectangular matrix-vector products.
\newblock In {\em The 24th International Conference on Artificial Intelligence
  and Statistics, {AISTATS}}, volume 130 of {\em Proceedings of Machine
  Learning Research}, pages 2755--2763, 2021.

\bibitem{ross2021conformal}
Brendan~Leigh Ross and Jesse~C Cresswell.
\newblock Tractable density estimation on learned manifolds with conformal
  embedding flows.
\newblock In {\em Advances in Neural Information Processing Systems}, 2021.

\bibitem{realnpv}
Laurent Dinh, Jascha Sohl{-}Dickstein, and Samy Bengio.
\newblock Density estimation using real {NVP}.
\newblock In {\em 5th International Conference on Learning Representations,
  {ICLR}}, 2017.

\bibitem{huan2018neural-autoregresive-flows}
Chin{-}Wei Huang, David Krueger, Alexandre Lacoste, and Aaron~C. Courville.
\newblock Neural autoregressive flows.
\newblock In {\em Proceedings of the 35th International Conference on Machine
  Learning, {ICML}}, volume~80 of {\em Proceedings of Machine Learning
  Research}, pages 2083--2092, 2018.

\bibitem{jaini2019sspolyflow}
Priyank Jaini, Kira~A. Selby, and Yaoliang Yu.
\newblock Sum-of-squares polynomial flow.
\newblock In {\em Proceedings of the 36th International Conference on Machine
  Learning, {ICML}}, volume~97 of {\em Proceedings of Machine Learning
  Research}, pages 3009--3018, 2019.

\bibitem{behrmann2019iresnet}
Jens Behrmann, Will Grathwohl, Ricky T.~Q. Chen, David Duvenaud, and
  J{\"{o}}rn{-}Henrik Jacobsen.
\newblock Invertible residual networks.
\newblock In {\em Proceedings of the 36th International Conference on Machine
  Learning, {ICML}}, volume~97 of {\em Proceedings of Machine Learning
  Research}, pages 573--582, 2019.

\bibitem{chen2018neuralode}
Tian~Qi Chen, Yulia Rubanova, Jesse Bettencourt, and David Duvenaud.
\newblock Neural ordinary differential equations.
\newblock In {\em Advances in Neural Information Processing Systems}, pages
  6572--6583, 2018.

\bibitem{maf}
George Papamakarios, Iain Murray, and Theo Pavlakou.
\newblock Masked autoregressive flow for density estimation.
\newblock In {\em Advances in Neural Information Processing Systems}, pages
  2338--2347, 2017.

\bibitem{dinh205nice}
Laurent Dinh, David Krueger, and Yoshua Bengio.
\newblock {NICE:} non-linear independent components estimation.
\newblock In {\em 3rd International Conference on Learning Representations,
  {ICLR} , Workshop Track Proceedings}, 2015.

\bibitem{kingma2016iaf}
Diederik~P. Kingma, Tim Salimans, Rafal J{\'{o}}zefowicz, Xi~Chen, Ilya
  Sutskever, and Max Welling.
\newblock Improving variational autoencoders with inverse autoregressive flow.
\newblock In {\em Advances in Neural Information Processing Systems}, pages
  4736--4744, 2016.

\bibitem{kingma2018glow}
Diederik~P. Kingma and Prafulla Dhariwal.
\newblock Glow: Generative flow with invertible 1x1 convolutions.
\newblock In {\em Advances in Neural Information Processing Systems}, pages
  10236--10245, 2018.

\bibitem{ho2019flow++}
Jonathan Ho, Xi~Chen, Aravind Srinivas, Yan Duan, and Pieter Abbeel.
\newblock Flow++: Improving flow-based generative models with variational
  dequantization and architecture design.
\newblock In {\em Proceedings of the 36th International Conference on Machine
  Learning, {ICML}}, volume~97 of {\em Proceedings of Machine Learning
  Research}, pages 2722--2730, 2019.

\bibitem{durkan2019nsf}
Conor Durkan, Artur Bekasov, Iain Murray, and George Papamakarios.
\newblock Neural spline flows.
\newblock In {\em Advances in Neural Information Processing Systems}, pages
  7509--7520, 2019.

\bibitem{dai2018diagnosing}
Bin Dai and David~P. Wipf.
\newblock Diagnosing and enhancing {VAE} models.
\newblock In {\em 7th International Conference on Learning Representations,
  {ICLR}}, 2019.

\bibitem{behrman2021explodinginverse}
Jens Behrmann, Paul Vicol, Kuan{-}Chieh Wang, Roger~B. Grosse, and
  J{\"{o}}rn{-}Henrik Jacobsen.
\newblock Understanding and mitigating exploding inverses in invertible neural
  networks.
\newblock In {\em The 24th International Conference on Artificial Intelligence
  and Statistics, {AISTATS}}, volume 130 of {\em Proceedings of Machine
  Learning Research}, pages 1792--1800, 2021.

\bibitem{brehmer2020mflow}
Johann Brehmer and Kyle Cranmer.
\newblock Flows for simultaneous manifold learning and density estimation.
\newblock In {\em Advances in Neural Information Processing Systems}, 2020.

\bibitem{cunningham2020normalizing}
Edmond Cunningham, Renos Zabounidis, Abhinav Agrawal, Ina Fiterau, and Daniel
  Sheldon.
\newblock Normalizing flows across dimensions, 2020.

\bibitem{kothari2021trumpets}
Konik Kothari, AmirEhsan Khorashadizadeh, Maarten~V. de~Hoop, and Ivan
  Dokmanic.
\newblock Trumpets: Injective flows for inference and inverse problems.
\newblock In {\em Proceedings of the Thirty-Seventh Conference on Uncertainty
  in Artificial Intelligence, {UAI}}, volume 161 of {\em Proceedings of Machine
  Learning Research}, pages 1269--1278, 2021.

\bibitem{kumar2020regularized}
Abhishek Kumar, Ben Poole, and Kevin Murphy.
\newblock Regularized autoencoders via relaxed injective probability flow.
\newblock In {\em The 23rd International Conference on Artificial Intelligence
  and Statistics, {AISTATS}}, volume 108 of {\em Proceedings of Machine
  Learning Research}, pages 4292--4301, 2020.

\bibitem{caterini2021rectangular}
Anthony~L. Caterini, Gabriel Loaiza-Ganem, Geoff Pleiss, and John~Patrick
  Cunningham.
\newblock Rectangular flows for manifold learning.
\newblock In {\em Advances in Neural Information Processing Systems},
  volume~34, pages 30228--30241, 2021.

\bibitem{shakir2021flowreview}
George Papamakarios, Eric~T. Nalisnick, Danilo~Jimenez Rezende, Shakir Mohamed,
  and Balaji Lakshminarayanan.
\newblock Normalizing flows for probabilistic modeling and inference.
\newblock {\em J. Mach. Learn. Res.}, 22:57:1--57:64, 2021.

\bibitem{dupont2019augmented}
Emilien Dupont, Arnaud Doucet, and Yee~Whye Teh.
\newblock Augmented neural odes.
\newblock In {\em Advances in Neural Information Processing Systems}, pages
  3134--3144, 2019.

\bibitem{cornish2020relaxing}
Robert Cornish, Anthony~L. Caterini, George Deligiannidis, and Arnaud Doucet.
\newblock Relaxing bijectivity constraints with continuously indexed
  normalising flows.
\newblock In {\em Proceedings of the 37th International Conference on Machine
  Learning, {ICML}}, volume 119 of {\em Proceedings of Machine Learning
  Research}, pages 2133--2143, 2020.

\bibitem{gemici2016normalizing}
Mevlana~C Gemici, Danilo Rezende, and Shakir Mohamed.
\newblock Normalizing flows on riemannian manifolds.
\newblock {\em arXiv preprint arXiv:1611.02304}, 2016.

\bibitem{mathieu2020riemannian}
Emile Mathieu and Maximilian Nickel.
\newblock Riemannian continuous normalizing flows.
\newblock {\em Advances in Neural Information Processing Systems},
  33:2503--2515, 2020.

\bibitem{dinh2019rad}
Laurent Dinh, Jascha Sohl{-}Dickstein, Razvan Pascanu, and Hugo Larochelle.
\newblock A {RAD} approach to deep mixture models.
\newblock In {\em Deep Generative Models for Highly Structured Data, {ICLR}},
  2019.

\bibitem{lee2006riemannian}
John~M Lee.
\newblock {\em Riemannian manifolds: an introduction to curvature}, volume 176.
\newblock Springer Science \& Business Media, 2006.

\bibitem{gallot1990riemannian}
Sylvestre Gallot, Dominique Hulin, and Jacques Lafontaine.
\newblock {\em Riemannian geometry}, volume~2.
\newblock Springer, 1990.

\bibitem{vqvae}
A{\"{a}}ron van~den Oord, Oriol Vinyals, and Koray Kavukcuoglu.
\newblock Neural discrete representation learning.
\newblock In {\em Advances in Neural Information Processing Systems}, pages
  6306--6315, 2017.

\bibitem{mittal2022autosdf}
Paritosh Mittal, Yen-Chi Cheng, Maneesh Singh, and Shubham Tulsiani.
\newblock Auto{SDF}: Shape priors for {3D} completion, reconstruction and
  generation.
\newblock In {\em Proceedings of the IEEE/CVF Conference on Computer Vision and
  Pattern Recognition}, pages 306--315, 2022.

\bibitem{lu2020conditional}
You Lu and Bert Huang.
\newblock Structured output learning with conditional generative flows.
\newblock In {\em The Thirty-Fourth {AAAI} Conference on Artificial
  Intelligence, {AAAI}}, pages 5005--5012, 2020.

\bibitem{deng2012the}
Li~Deng.
\newblock The mnist database of handwritten digit images for machine learning
  research.
\newblock {\em IEEE Signal Processing Magazine}, 29(6):141--142, 2012.

\end{thebibliography}
\end{document}


\maketitle
This document presents the discussions and results left out in the main paper due to space constraints. We begin with the details regarding the considered data distributions. We then present quantitative and qualitative results left out from the main paper. We finally conclude with further information regarding the experiments and implementation. 

\section{Data Generation}
We generated and experimented with ten 3-dimensional data distributions over manifolds of varying complexity. Figure \ref{3D-data} provides the visualizations for each of the considered datasets. We elaborate more on the equations used to generate data from each of these distribution below.
\begin{figure*}[h!]
    \centering
    \begin{tabular}{cc}
    \subfloat[Spherical]{\includegraphics[width = 0.20\linewidth]{figures/datasets/realdata_SPHERICAL.png}}
    \subfloat[Helix]{\includegraphics[width = 0.20\linewidth]{figures/datasets/realdata_HELIX.png}}
    \subfloat[Lissajous]{\includegraphics[width = 0.20\linewidth]{figures/datasets/realdata_LISSAJOUS.png}}
    \subfloat[Twisted-Eight]{\includegraphics[width = 0.20\linewidth]{figures/datasets/realdata_TwistedEIGHT.png}}
    \subfloat[Knotted]{\includegraphics[width = 0.20\linewidth]{figures/datasets/realdata_KNOTTED.png}}\\
    \subfloat[InterlockedCircles]{\includegraphics[width = 0.20\linewidth]{figures/datasets/realdata_InterlockedCIRCLES.png}}
    \subfloat[Non-Knotted]{\includegraphics[width = 0.20\linewidth]{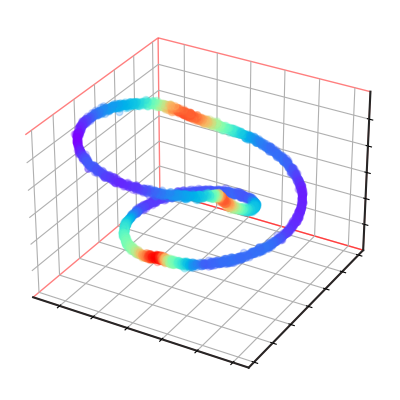}}
    \subfloat[Bent-Lissajous]{\includegraphics[width = 0.20\linewidth]{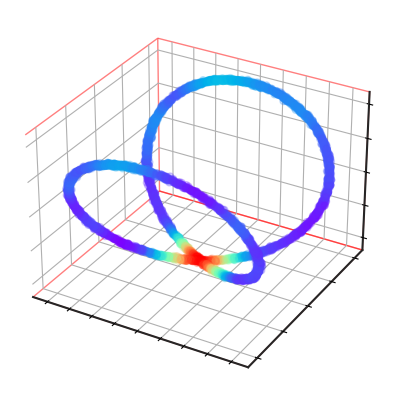}}
    \subfloat[Disjoint-Circles]{\includegraphics[width = 0.20\linewidth]{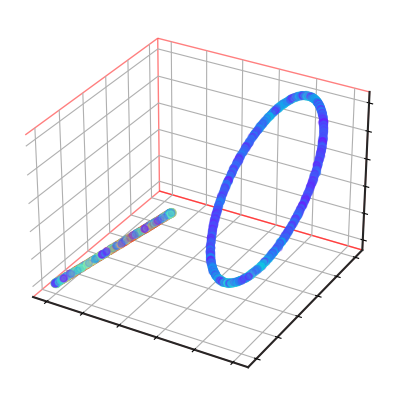}}
    \subfloat[Star]{\includegraphics[width = 0.20\linewidth]{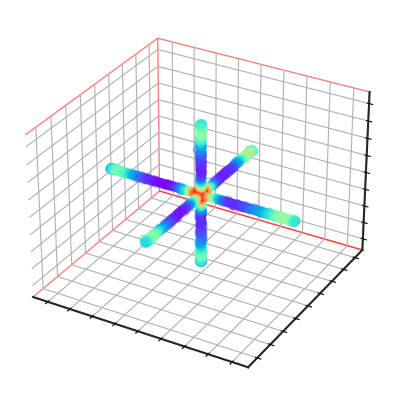}}
    \end{tabular}
    \caption{Visualizations of the considered 3-dimensional data distributions}
    \label{3D-data}
\end{figure*}

\subsection{Spherical}
We considered three mixture of Gaussians in the 3-dimensional space with parameters $(\mu_1,\sigma)$,$(\mu_2,\sigma)$,$(\mu_3,\sigma)$ respectively. Samples ($x$) drawn uniformly from each of the three Gaussians were then projected on to the unit sphere in 3D as $x / ||x||$. The means ($\mu_i$) were sampled from a standard normal distribution and the standard deviation was set to 0.2. The exact parameter values used are: $\mu_1=(-0.15,-0.77, 0.94)$, $\mu_2=(0.79,-0.75,-0.02)$, $\mu_3=(0.04,0.40,1.31)$ and $\sigma=0.2$.

\subsection{Helix}
To generate the Helix data distribution, we first sample $\theta \in \R$ uniformly from $[0,8\pi]$. For each $\theta$, we then generate the datapoint $\mathbf{x}=(x,y,z)+\epsilon$, where $\epsilon \sim N(0,\sigma=0.01)$ and $-$
\begin{itemize*}
    \item $x = \theta$
    \item $y = \cos\theta$
    \item $z = \sin\theta$
\end{itemize*}

\subsection{Lissajous}
To generate the Lissajous data distribution, we first sample $\theta \in \R$ uniformly from $[-\pi,\pi]$. For each $\theta$, we then generate the datapoint $\mathbf{x}=(x,y,z)+\epsilon$, where $\epsilon \sim N(0,\sigma=0.01)$ and $-$ 
\begin{itemize*}
    \item $x = \cos\theta$
    \item $y = 0$
    \item $z = \sin(2\theta)$
\end{itemize*}

\subsection{Twisted-Eight}
To generate the Twisted-Eight data distribution, we sample $\theta \in \R$ uniformly from $[-\pi,\pi]$. For each $\theta$, we then generate two datapoints $\mathbf{x_1}=(x_1,y_1,z_1)+\epsilon$ and $\mathbf{x_2}=(x_2,y_2,z_2)+\epsilon$, where $\epsilon \sim N(0,\sigma=0.01)$ and $-$ 
\begin{itemize*}
    \item $x_1 = \sin\theta$
    \item $y_1 = \cos\theta$
    \item $z_1 = 0$
\end{itemize*}.  
\\
\begin{itemize*}
    \item $x_2 = 2+\sin\theta$
    \item $y_2 = 0$
    \item $z_2 = \cos\theta$
\end{itemize*}\\
The final distribution is the union of the distributions over $\mathbf{x_1}$ and $\mathbf{x_2}$.

\subsection{Knotted}
To generate the Knotted data distribution, we first sample $\theta \in \R$ uniformly from $[-\pi,\pi]$. For each $\theta$, we then generate the datapoint $\mathbf{x}=(x,y,z)+\epsilon$, where $\epsilon \sim N(0,\sigma=0.01)$ and $-$ 
\begin{itemize*}
    \item $x = \sin\theta + 2\sin2\theta$
    \item $y = \cos\theta - 2\cos2\theta$
    \item $z = \sin3\theta$
\end{itemize*}

\subsection{Interlocked-Circles}
To generate the Interlocked-Circles data distribution, we sample $\theta \in \R$ uniformly from $[-\pi,\pi]$. For each $\theta$, we then generate  two datapoints $\mathbf{x_1}=(x_1,y_1,z_1)+\epsilon$ and $\mathbf{x_2}=(x_2,y_2,z_2)+\epsilon$, where $\epsilon \sim N(0,\sigma=0.01)$ and $-$ 
\begin{itemize*}
    \item $x_1 = \sin\theta$
    \item $y_1 = \cos\theta$
    \item $z_1 = 0$
\end{itemize*}.  
\\
\begin{itemize*}
    \item $x_2 = 1+\sin\theta$
    \item $y_2 = 0$
    \item $z_2 = \cos\theta$
\end{itemize*}\\
The final distribution is the union of the distributions over $\mathbf{x_1}$ and $\mathbf{x_2}$.

\subsection{Non-Knotted}
To generate the Non-Knotted data distribution, we first sample $\theta \in \R$ uniformly from $[-\pi,\pi]$. For each $\theta$, we then generate the datapoint $\mathbf{x}=(x,y,z)+\epsilon$, where $\epsilon \sim N(0,\sigma=0.01)$ and $-$
\begin{itemize}
    \item $x = (1 + 0.5\cos3\theta) \cos2\theta$
    \item $y = (1 + 0.5\cos3\theta) \sin2\theta$
    \item $z = 0.5\sin\theta$
\end{itemize}

\subsection{Bent-Lissajous}
To generate the Bent-Lissajous data distribution, we first sample $\theta \in \R$ uniformly from $[-\pi,\pi]$. For each $\theta$, we then generate the datapoint $\mathbf{x}=(x,y,z)+\epsilon$, where $\epsilon \sim N(0,\sigma=0.01)$ and $-$
\begin{itemize*}
    \item $x = \sin2\theta$
    \item $y = \cos\theta$
    \item $z = \cos2\theta$
\end{itemize*}

\subsection{Disjoint-Circles}
To generate the Disjoint-Circles data distribution, we sample $\theta \in \R$ uniformly from $[-\pi,\pi]$. For each $\theta$, we then generate  two datapoints $\mathbf{x_1}=(x_1,y_1,z_1)+\epsilon$ and $\mathbf{x_2}=(x_2,y_2,z_2)+\epsilon$, where $\epsilon \sim N(0,\sigma=0.01)$ and $-$ 
\begin{itemize*}
    \item $x_1 = -1 + \sin\theta$
    \item $y_1 = -1 + \sin\theta$
    \item $z_1 = -1 + \sin\theta$
\end{itemize*}.  
\\
\begin{itemize*}
    \item $x_2 = 2+\sin\theta$
    \item $y_2 = 1 + 2\cos\theta$
    \item $z_2 = 1 + 2\cos\theta$
\end{itemize*}\\
The final distribution is the union of the distributions over $\mathbf{x_1}$ and $\mathbf{x_2}$.

\subsection{Star}
To generate the Star data distribution, we sample $\theta \in \R$ uniformly from $[-\pi,\pi]$. For each $\theta$, we then generate three datapoints $\mathbf{x_1}=(x_1,y_1,z_1)+\epsilon$, $\mathbf{x_2}=(x_2,y_2,z_2)+\epsilon$, and  $\mathbf{x_3}=(x_3,y_3,z_3)+\epsilon$ where $\epsilon \sim N(0,\sigma=0.01)$ and $-$ 
\begin{itemize*}
    \item $x_1 = \sin\theta$
    \item $y_1 = 0 $
    \item $z_1 = 0$
\end{itemize*}.  
\\
\begin{itemize*}
    \item $x_2 = 0$
    \item $y_2 = \sin\theta$
    \item $z_2 = 0$
\end{itemize*}
\\\begin{itemize*}
    \item $x_3 = 0$
    \item $y_3 = 0$
    \item $z_3 = \sin\theta$
\end{itemize*}\\
The final distribution is the union of the distributions over $\mathbf{x_1}$, $\mathbf{x_2}$ and $\mathbf{x_3}$.

\section{Additional Results}
\subsection{Density Estimation and Sample Generation}
We provide quantitative evaluations for density estimation and sample generation over four additional 3-dimensional data distribution discussed Section 1 in Table \ref{density-est-quant} and Table \ref{sample-gen-quantitative} respectively. We can observe that the models trained with the augmentation of our framework achieves better performance for both density estimation and sample generation than their corresponding baselines. We also validate and compare the goodness of the generated samples through qualitative visualizations in Figures \ref{qual-sample-begin} - \ref{qual-sample-end}. Note that CEFs perform poorer than the other baselines because they consist of a 2-dimensional base flow which is then raised to the 3-dimensional space using a conformal embedding. The other flows (RealNVP and MAF) are, on the other hand, trained in the 3-dimensional space. A particularly interesting observation here is that the data distributions learned by CEF without VQ-augmentation tend to be planar in the 3-dimensional space. This demonstrates the limited expressivity of global conformal dimension raising transformations. The local conformal transformations obtained with the augmentation of our framework are, on the other hand, able to better capture the global structure of the data distribution and generate better samples. 

\subsection{Gaussianization}
The ability of a normalizing flow to generate high fidelity samples from given data distribution is also governed by whether the latent space learned through the flow transformation matches the assumed prior. For a flow with a Normal distribution assumed in the latent space, this means that the forward flow transformations should effectively  Gaussianize the given data distribution. In Figures \ref{gauss-begin} to \ref{gauss-end} we thus visualize and compare how different data distributions are transformed gradually by each layer of a RealNVP flow trained with and without the augmentation of our proposed framework. We can observe that the models trained with VQ-augmentation learn to better transform the input space to match the assumed prior. As a result, they are also able to generate better samples.

\section{Implementation Details}

\begin{figure*}[h!]
    \centering
     \begin{tabular}{cc}     
    \subfloat[Forward Transformation]{\includegraphics[width=0.3\linewidth]{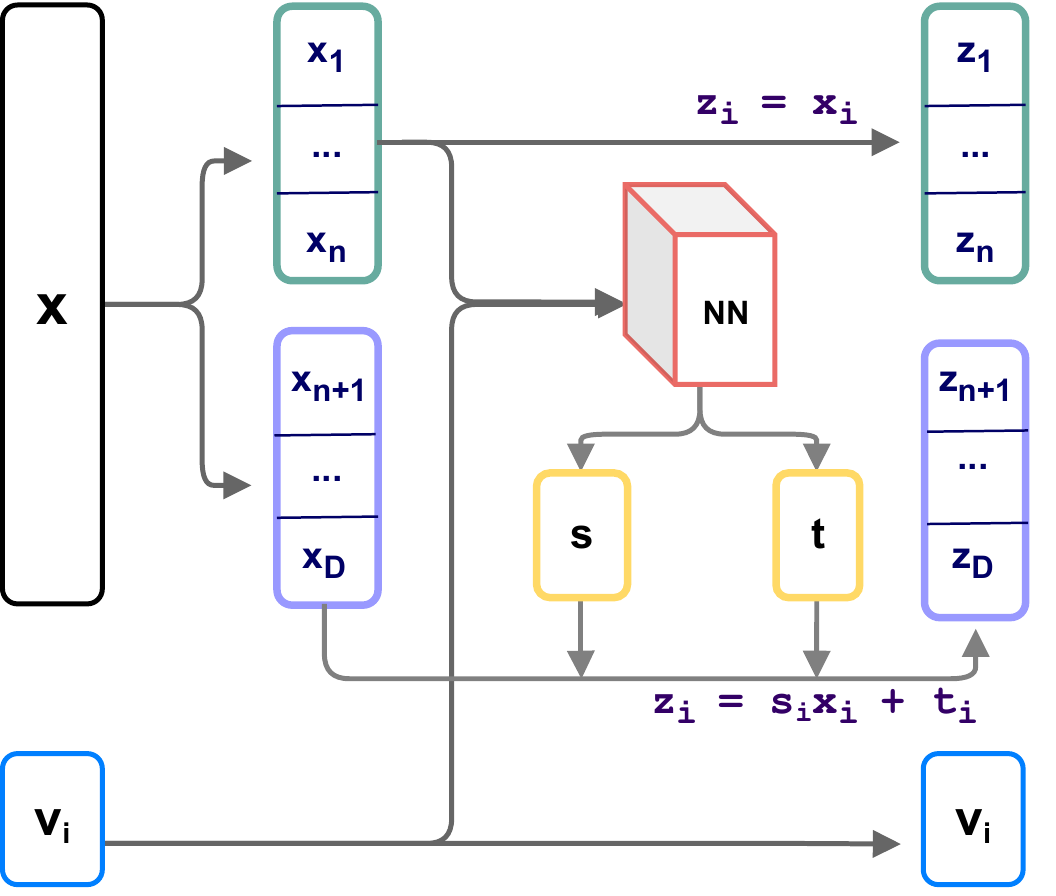}} \hspace{50pt}
    \subfloat[Inverse Transformation]{ \includegraphics[width=0.3\linewidth]{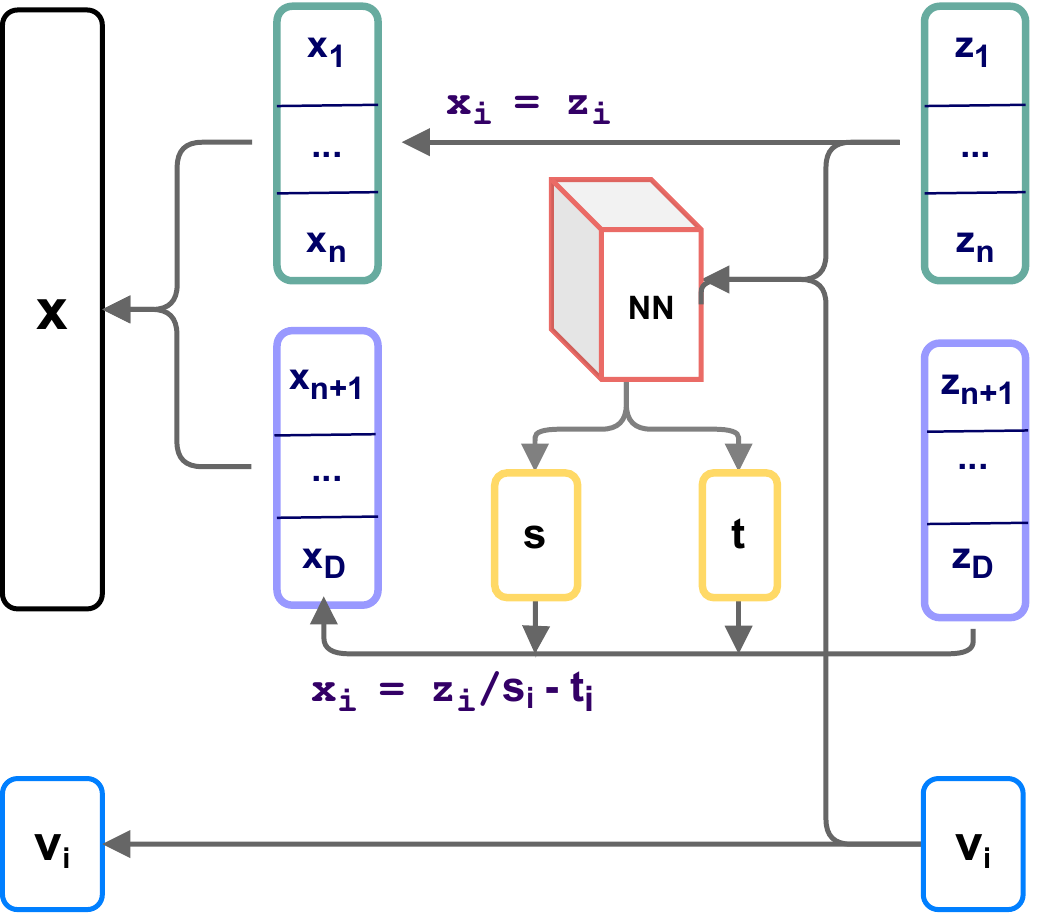}}
    \end{tabular}
    \caption{Parametrizing the conditional coupling layer transformation.}
    \label{fig:cond-transform}
\end{figure*}

To experimentally validate the efficacy of the proposed framework, we consider the ten datasets presented in Section 1. Each dataset consists of $10,000$ datapoints, $5,000$ of which we use for training and $2,500$ each for validation and testing. We train three different normalizing flows - RealNVP, Masked Autoregressive Flows (MAF) and Conformal Embedding Flows (CEF).  We define each model using $5$ flow transformations and train them for $100$ epochs using an Adam optimizer with a learning rate of $1e-4$ and a batch size of $128$. We follow the same hyperparameters for a base flow and its VQ-counterpart without any tuning and report the performance averaged over $5$ independent trials. 
 We early stop if the validation performance does not improve over 10 epochs. The architectural details pertaining to each of the models are given below: 

\textbf{RealNVP}- We compose the RealNVP flow using $5$ coupling layer transformations, each followed by a batch-normalization. We use feedforward networks with $2$ hidden layers, each consisting of $128$ hidden nodes as the non-linear transformation to obtain the scaling and translation parameters. We use $tanh$ as the activation function for the scale network and $relu$ as the activation function for the translation network.

\textbf{MAF}- We compose the MAF flow using $5$ masked auto-regressive layer transformations, each followed by a batch-normalization. In each layer, we use a masked feedforward network with $1$ hidden layer, consisting of $128$ hidden nodes. We use $relu$ as the activation function for the feedforward network.

\textbf{CEF}-
 We compose the CEF flow  using $5$ coupling layer transformations in 2-dimensional space, followed by the conformal transformation that raises the dimension to $3$. We use the same architecture reported above for RealNVP in the coupling transforms. The conformal embedding is parameterized as given in \cite{ross2021conformal}, using a composition of Scaling, Shifting, Orthogonal, Special Conformal and Padding transformations. As CEF is an injective flow, we follow \cite{ross2021conformal} and train it to minimize the reconstruction loss for 20 epochs prior to employing the maximum likelihood training.

\textbf{VQ-\textit{flow}}-
We parameterize the encoder and decoder of the VQ-AE using feedforward neural networks. In each network, we use $4$ hidden layers each consisting of $128$ hidden nodes followed by batch-normalization and a leaky$-relu$ activation with negative slope of $0.2$.
To learn the partitioning of the data manifold, we use a latent dimension of $2$ with $k=32$ learnable quantized centers. We train the VQ-AE for $50$ epochs to minimize the reconstruction loss using an Adam optimizer with a learning rate of $1e-4$ and batch size of $128$.

To define the conditional normalizing flow, we use the parameterization given in \cite{lu2020conditional,maf}. The key idea is to incorporate the quantized center as additional conditioning information to the unrestricted (non-invertible) neural network used in the coupling and auto-regressive transformations. Figure  \ref{fig:cond-transform} demonstrates the construction of such a conditional coupling layer transform. To define conditional conformal transformations, we use $k$ conformal embeddings and index into it using the quantized center. We believe that we can extend our framework to other arbitrary flows by adapting the conditional flow transformations defined in \cite{cornish2020relaxing}, which we leave to future work.

\begin{figure*}[h!]
    \centering
    \includegraphics[width=\linewidth]{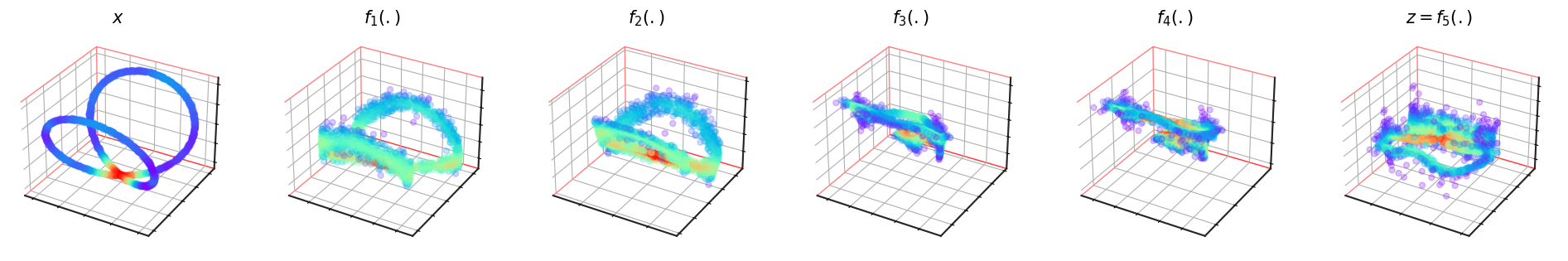}\\
    \includegraphics[width=\linewidth]{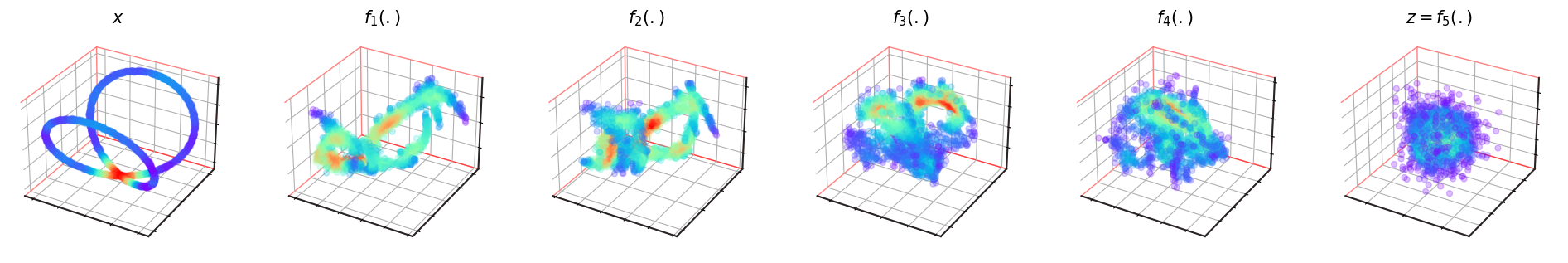}
    \caption{Visualization of the latent transformation achieved using RealNVP (Top Row) and VQ-RealNVP (Bottom Row) on the Bent-Lissajous data distribution.}
    \label{gauss-begin}
\end{figure*}

\begin{figure*}[h!]
    \centering
    \includegraphics[width=\linewidth]{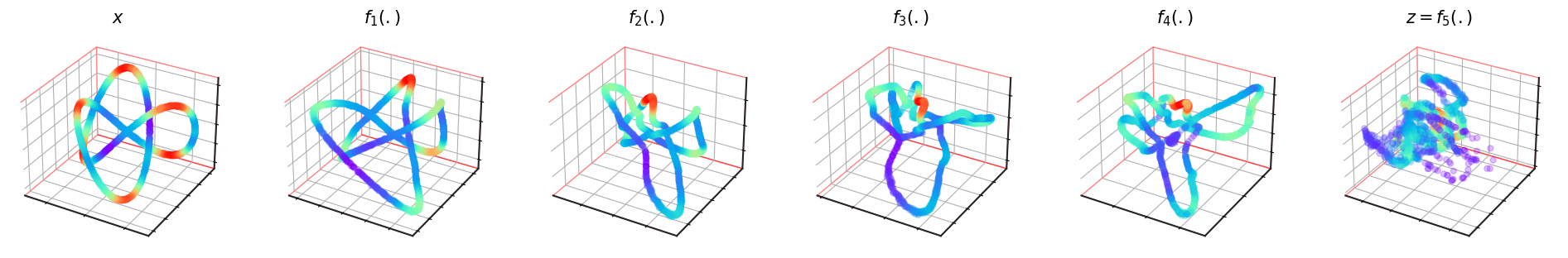}\\
    \includegraphics[width=\linewidth]{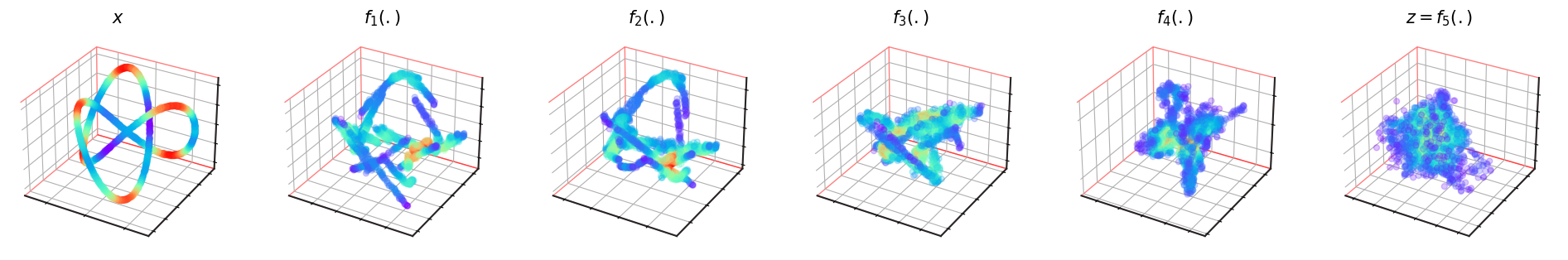}
    \caption{Visualization of the latent transformation achieved using RealNVP (Top Row) and VQ-RealNVP (Bottom Row) on the Knotted data distribution.}
\end{figure*}

\begin{figure*}[h!]
    \centering
    \includegraphics[width=\linewidth]{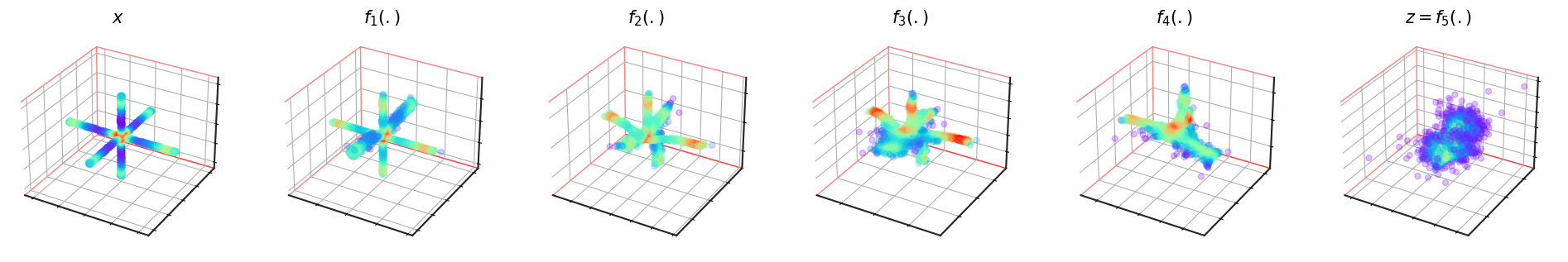}\\
    \includegraphics[width=\linewidth]{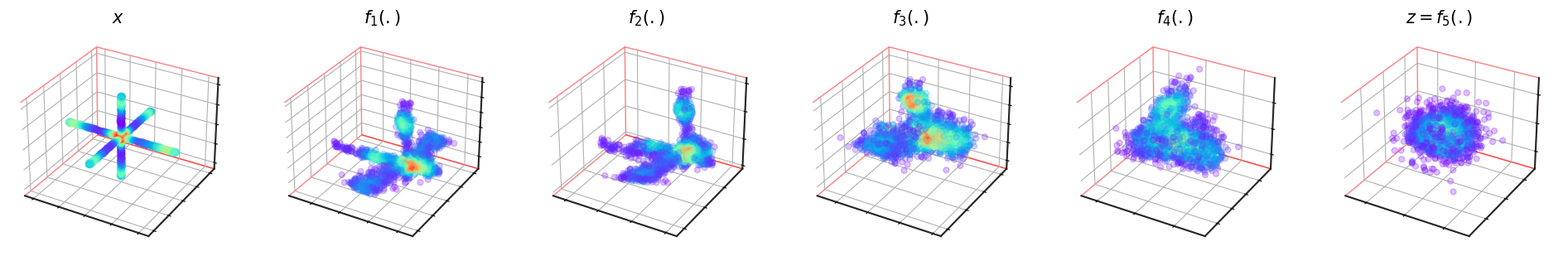}
    \caption{Visualization of the latent transformation achieved using RealNVP (Top Row) and VQ-RealNVP (Bottom Row) on the Star data distribution.}
\end{figure*}

\begin{figure*}[h!]
    \centering
    \includegraphics[width=\linewidth]{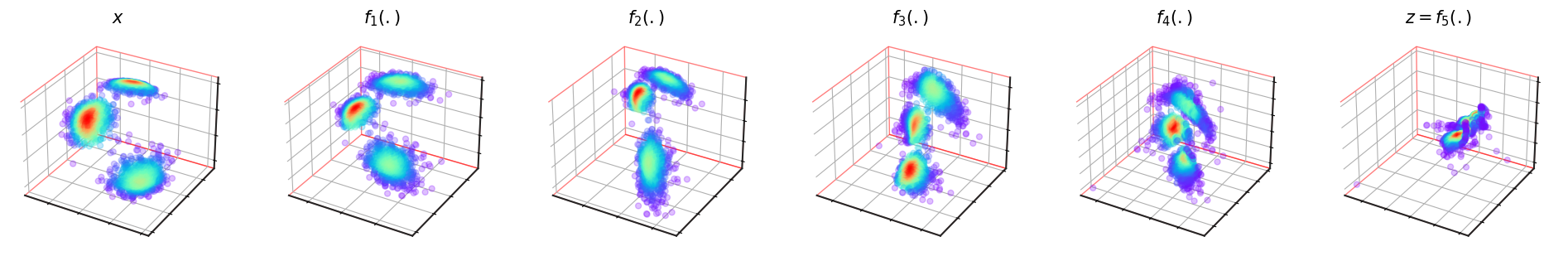}\\
    \includegraphics[width=\linewidth]{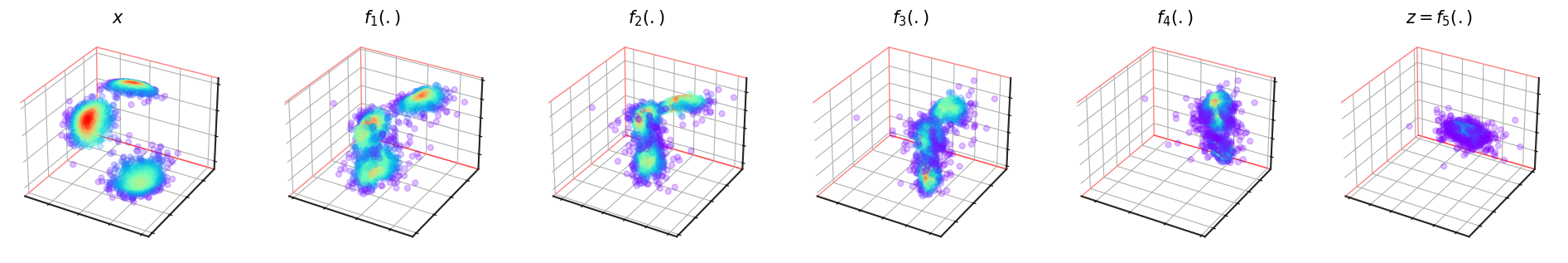}
    \caption{Visualization of the latent transformation achieved using RealNVP (Top Row) and VQ-RealNVP (Bottom Row) on the Spherical data distribution.}
     \label{gauss-end}
\end{figure*}

\begin{table*}[h!]
\centering
\begin{tabular}{@{}cccccc@{}}
\toprule
\textbf{Model} & Non-Knotted   & Bent-Lissajous  & Disjoint-Circles  & Star \\ \midrule
Real NVP       & 0.53 $\pm$ 0.18  & 1.04 $\pm$ 0.22   &  1.71 $\pm$ 0.12     & 3.33 $\pm$ 0.18            \\
VQ-RealNVP     & 2.39 $\pm$ 0.24  & 2.62 $\pm$ 0.13   &  2.71 $\pm$ 0.19     & 4.23 $\pm$ 0.06             \\ \midrule
MAF            & 0.73 $\pm$ 0.18  & 1.48 $\pm$ 0.11   &  1.95 $\pm$ 0.12     & 3.53 $\pm$ 0.03             \\
VQ-MAF         & 2.41 $\pm$ 0.19  & 2.06 $\pm$ 0.12   &  2.87 $\pm$ 0.07     & 3.59 $\pm$ 0.12             \\ \midrule
CEF            & -0.46 $\pm$ 0.13 & -0.51 $\pm$ 0.16  & -0.71 $\pm$ 0.21     & 1.26 $\pm$ 0.11             \\
VQ-CEF         & -0.15 $\pm$ 0.09 & -0.54 $\pm$ 0.22  & 0.24 $\pm$ 0.15      & 1.32 $\pm$ 0.02             \\ \bottomrule
\end{tabular}%
\caption{ Quantitative performance evaluation for \textbf{Density Estimation} in terms of the test log-likelihood in nats (higher the better) on the toy 3D Datasets. The values are averaged across 5 independent trials, $\pm$ represents the 95\% confidence interval.}
\label{density-est-quant}
\end{table*}

\begin{table*}[t!]
\centering
\begin{tabular}{@{}cccccl@{}}
\toprule
\textbf{Model} & Non-Knotted   & Bent-Lissajous & Disjoint-Circles  & Star \\ \midrule
Real NVP       & 0.53  $\pm$  0.18  & 1.04  $\pm$  0.22     &  1.71  $\pm$  0.12     & 3.33  $\pm$  0.18            \\
VQ-RealNVP     & 2.39  $\pm$  0.24  & 2.62  $\pm$  0.13     &  2.71  $\pm$  0.19     & 4.23  $\pm$  0.06             \\ \midrule
MAF            & 0.73  $\pm$  0.18  & 1.48  $\pm$  0.11     &  1.95  $\pm$  0.12     & 3.53  $\pm$  0.03             \\
VQ-MAF         & 2.41  $\pm$  0.19  & 2.06  $\pm$  0.12     &  2.87  $\pm$  0.07     & 3.59  $\pm$  0.12             \\ \midrule
CEF            & -0.46  $\pm$  0.13 & -0.51  $\pm$  0.16    & -0.71  $\pm$  0.21     & 1.26  $\pm$  0.11             \\
VQ-CEF         & -0.15  $\pm$  0.09 & -0.54  $\pm$  0.22    & 0.24  $\pm$  0.15      & 1.32  $\pm$  0.02             \\ \bottomrule
\end{tabular}
\caption{ Quantitative performance evaluation for \textbf{Sample Generation} in terms of the log-likelihood in nats (higher the better) on the toy 3D Datasets. The values are averaged across 5 independent trials, $\pm$ represents the 95\% confidence interval.}
\label{sample-gen-quantitative}
\end{table*}

\begin{figure*}[h!]
    \centering
    \begin{tabular}{cc}     \subfloat[InterlockedCircles]{\includegraphics[width = 0.2\linewidth]{figures/datasets/realdata_InterlockedCIRCLES.png}}
    \subfloat[Bent-Lissajous]{\includegraphics[width = 0.2\linewidth]{figures/datasets/realdata_BentLISSAJOUS.png}}
    \subfloat[Non-Knotted]{\includegraphics[width = 0.2\linewidth]{figures/datasets/realdata_NonKNOTTED.png}}
    \subfloat[Disjoint-Circles]{\includegraphics[width = 0.2\linewidth]{figures/datasets/realdata_DisjointCIRCLES.png}}
    \subfloat[Star]{\includegraphics[width = 0.2\linewidth]{figures/datasets/realdata_STAR.png}}
    \end{tabular}\\
    \includegraphics[width = 0.195\linewidth]{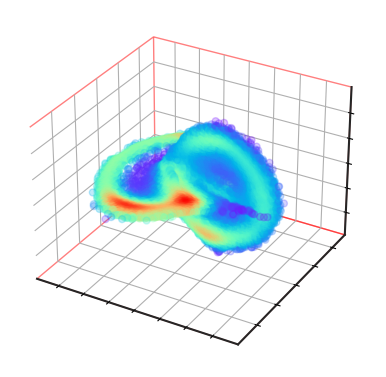}
    \includegraphics[width = 0.195\linewidth]{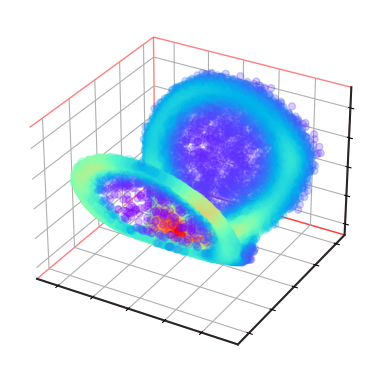}
    \includegraphics[width = 0.195\linewidth]{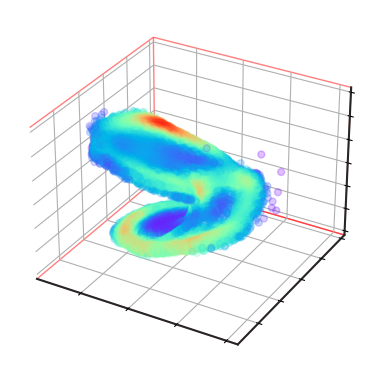}
    \includegraphics[width = 0.195\linewidth]{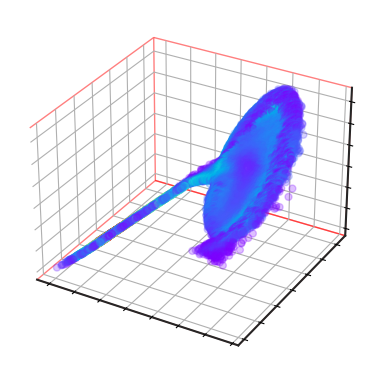}
    \includegraphics[width = 0.195\linewidth]{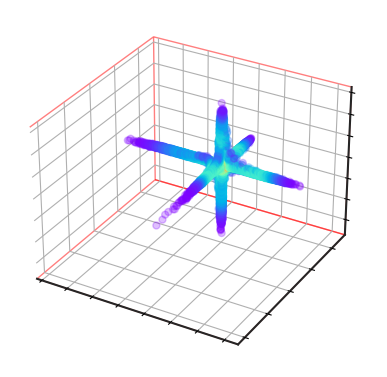} \\ 
    \includegraphics[width = 0.195\linewidth]{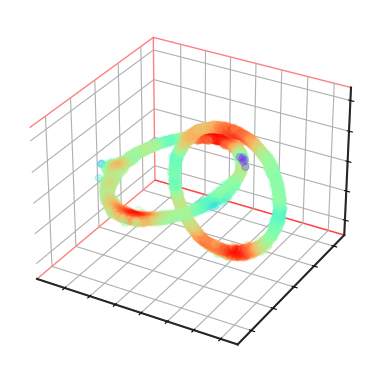} 
    \includegraphics[width = 0.195\linewidth]{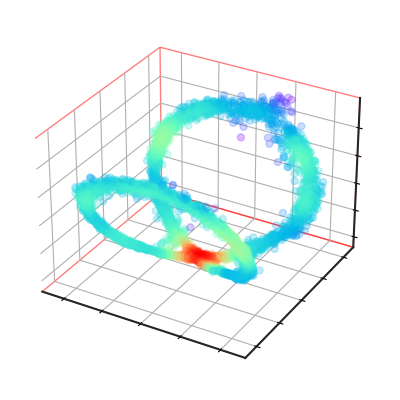}
    \includegraphics[width = 0.195\linewidth]{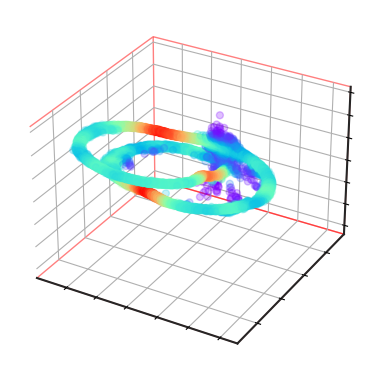}
    \includegraphics[width = 0.195\linewidth]{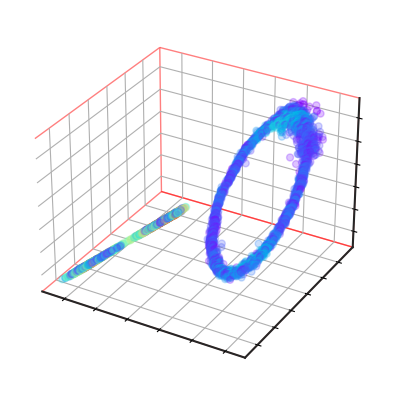}
    \includegraphics[width = 0.195\linewidth]{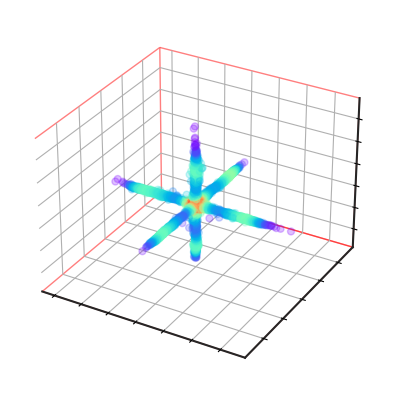}
    \caption{Qualitative visualization of the samples generated by a classical flow - \textbf{RealNVP} (Middle Row) and its VQ-counterpart (Bottom Row) trained on Toy 3D data distributions (Top Row).}
    \label{qual-sample-begin}
\end{figure*}

\begin{figure*}[h!]
    \centering
    \begin{tabular}{ccccc}
    \subfloat[Spherical]{\includegraphics[width = 0.2\linewidth]{figures/datasets/realdata_SPHERICAL.png}}
    \subfloat[Helix]{\includegraphics[width = 0.2\linewidth]{figures/datasets/realdata_HELIX.png}}
    \subfloat[Lissajous]{\includegraphics[width = 0.2\linewidth]{figures/datasets/realdata_LISSAJOUS.png}}
    \subfloat[Twisted-Eight]{\includegraphics[width = 0.2\linewidth]{figures/datasets/realdata_TwistedEIGHT.png}}
    \subfloat[Knotted]{\includegraphics[width = 0.2\linewidth]{figures/datasets/realdata_KNOTTED.png}}
    \end{tabular}\\
    \includegraphics[width = 0.195\linewidth]{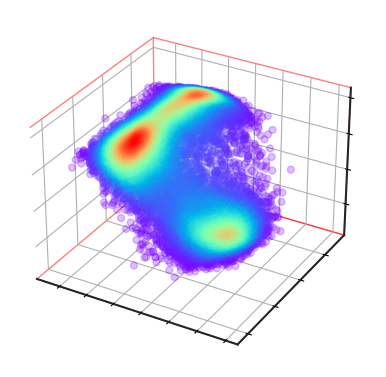}
    \includegraphics[width = 0.195\linewidth]{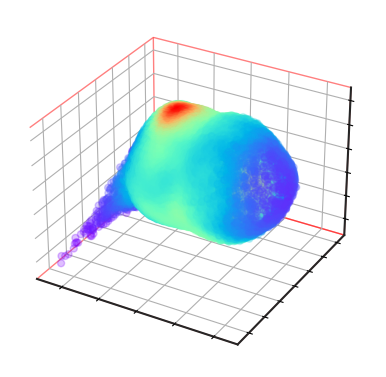}
    \includegraphics[width = 0.195\linewidth]{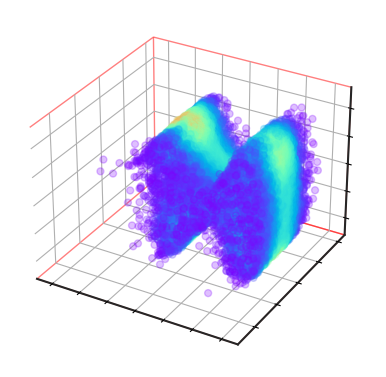}
    \includegraphics[width = 0.195\linewidth]{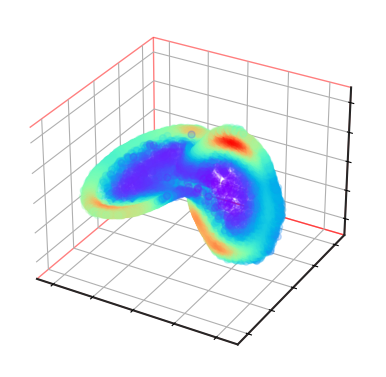}
    \includegraphics[width = 0.195\linewidth]{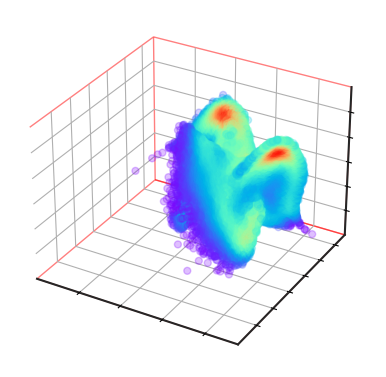}\\
    \includegraphics[width = 0.195\linewidth]{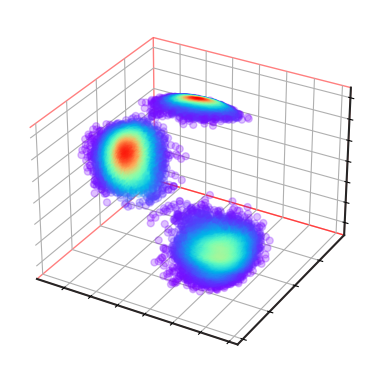}
    \includegraphics[width = 0.195\linewidth]{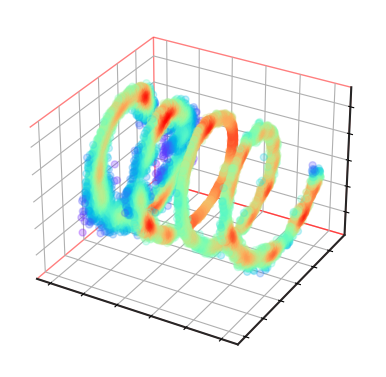}
    \includegraphics[width = 0.195\linewidth]{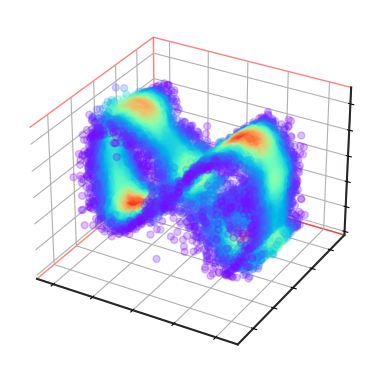}
    \includegraphics[width = 0.195\linewidth]{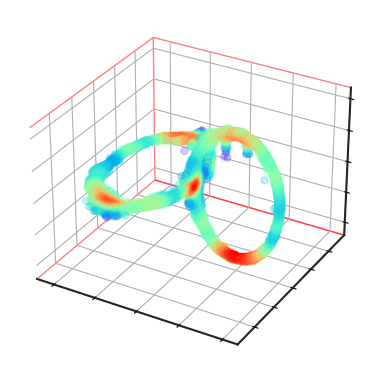}
    \includegraphics[width = 0.195\linewidth]{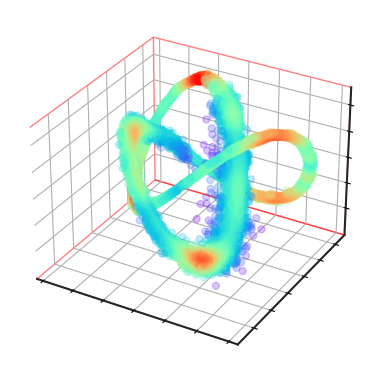}
    \caption{Qualitative visualization of the samples generated by a classical flow - \textbf{MAF} (Middle Row) and its VQ-counterpart (Bottom Row) trained on Toy 3D data distributions (Top Row).}
\end{figure*}

\begin{figure*}[h!]
    \centering
    \begin{tabular}{cc}     \subfloat[InterlockedCircles]{\includegraphics[width = 0.2\linewidth]{figures/datasets/realdata_InterlockedCIRCLES.png}}
    \subfloat[Bent-Lissajous]{\includegraphics[width = 0.2\linewidth]{figures/datasets/realdata_BentLISSAJOUS.png}}
    \subfloat[Non-Knotted]{\includegraphics[width = 0.2\linewidth]{figures/datasets/realdata_NonKNOTTED.png}}
    \subfloat[Disjoint-Circles]{\includegraphics[width = 0.2\linewidth]{figures/datasets/realdata_DisjointCIRCLES.png}}
    \subfloat[Star]{\includegraphics[width = 0.2\linewidth]{figures/datasets/realdata_STAR.png}}
    \end{tabular}\\
    \includegraphics[width = 0.195\linewidth]{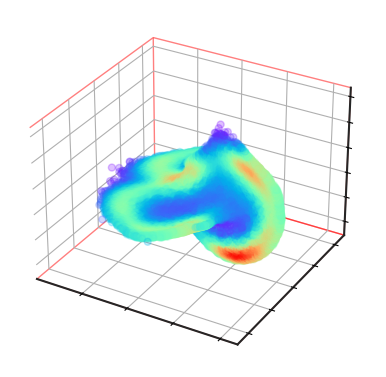}
    \includegraphics[width = 0.195\linewidth]{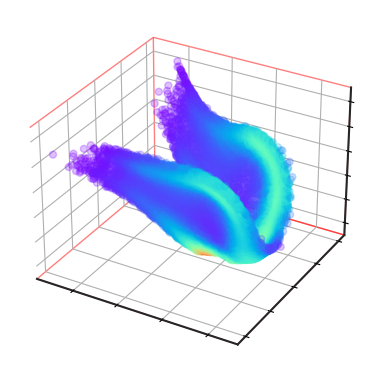}
    \includegraphics[width = 0.195\linewidth]{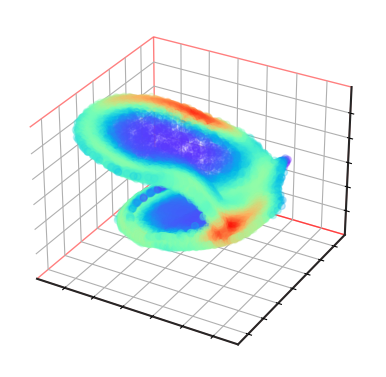}
    \includegraphics[width = 0.195\linewidth]{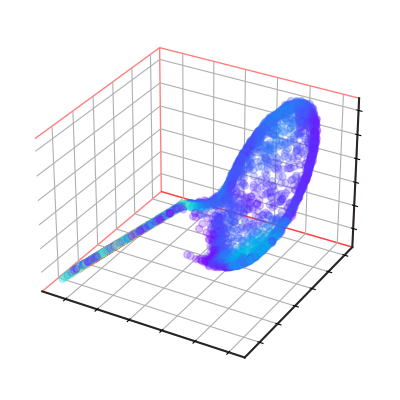}
    \includegraphics[width = 0.195\linewidth]{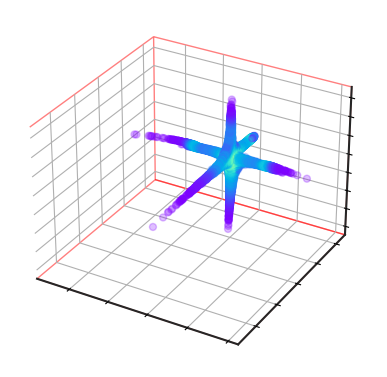} \\ 
    \includegraphics[width = 0.195\linewidth]{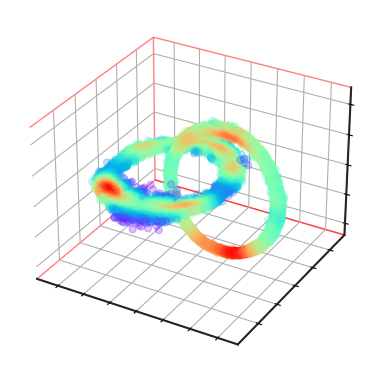} 
    \includegraphics[width = 0.195\linewidth]{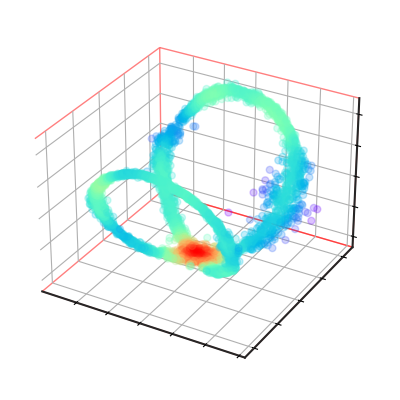}
    \includegraphics[width = 0.195\linewidth]{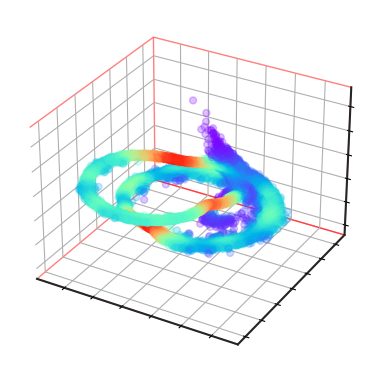}
    \includegraphics[width = 0.195\linewidth]{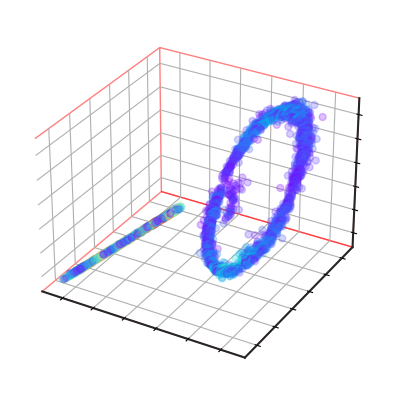}
    \includegraphics[width = 0.195\linewidth]{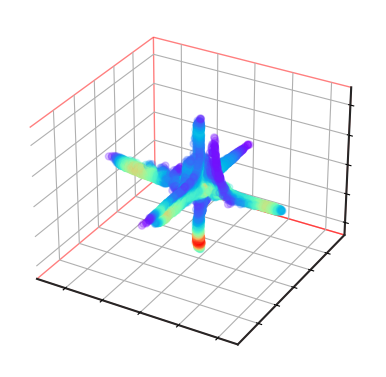}
    \caption{Qualitative visualization of the samples generated by a classical flow - \textbf{MAF} (Middle Row) and its VQ-counterpart (Bottom Row) trained on Toy 3D data distributions (Top Row).}
\end{figure*}


\begin{figure*}[h!]
    \centering
    \begin{tabular}{ccccc}
    \subfloat[Spherical]{\includegraphics[width = 0.2\linewidth]{figures/datasets/realdata_SPHERICAL.png}}
    \subfloat[Helix]{\includegraphics[width = 0.2\linewidth]{figures/datasets/realdata_HELIX.png}}
    \subfloat[Lissajous]{\includegraphics[width = 0.2\linewidth]{figures/datasets/realdata_LISSAJOUS.png}}
    \subfloat[Twisted-Eight]{\includegraphics[width = 0.2\linewidth]{figures/datasets/realdata_TwistedEIGHT.png}}
    \subfloat[Knotted]{\includegraphics[width = 0.2\linewidth]{figures/datasets/realdata_KNOTTED.png}}
    \end{tabular}\\
    \includegraphics[width = 0.195\linewidth]{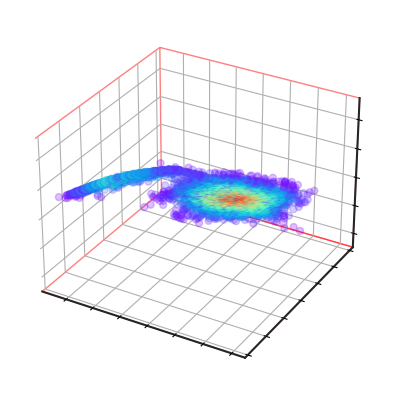}
    \includegraphics[width = 0.195\linewidth]{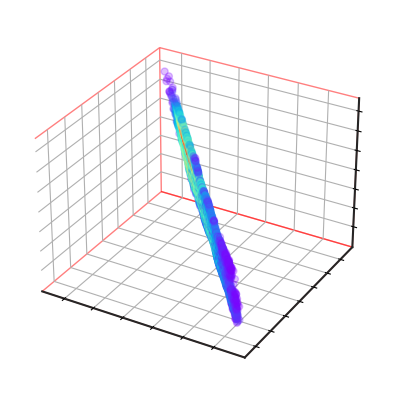}
    \includegraphics[width = 0.195\linewidth]{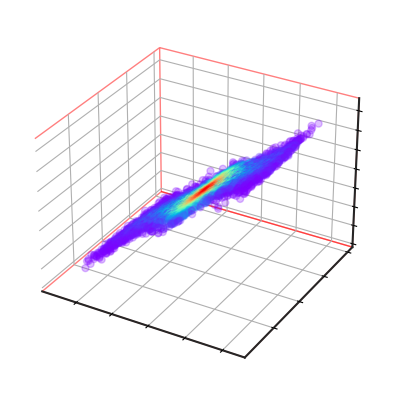}
    \includegraphics[width = 0.195\linewidth]{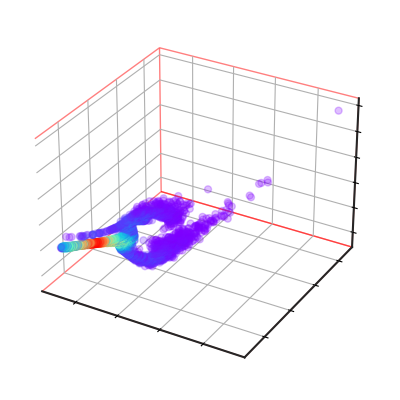}
    \includegraphics[width = 0.195\linewidth]{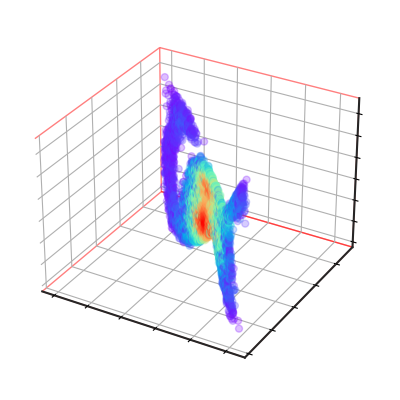}\\
    \includegraphics[width = 0.195\linewidth]{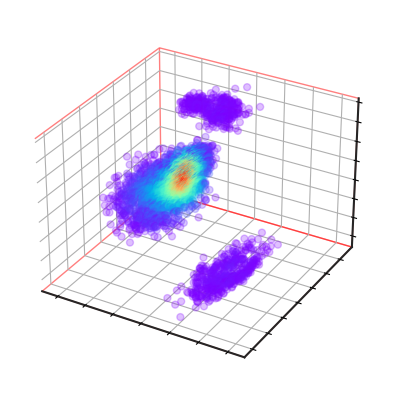}
    \includegraphics[width = 0.195\linewidth]{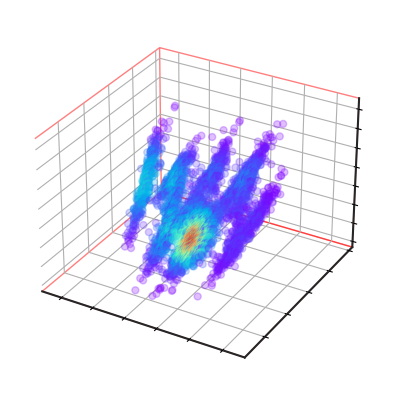}
    \includegraphics[width = 0.195\linewidth]{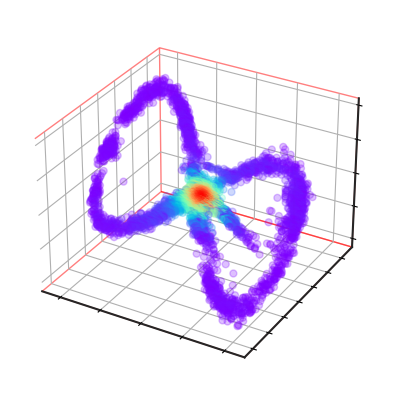}
    \includegraphics[width = 0.195\linewidth]{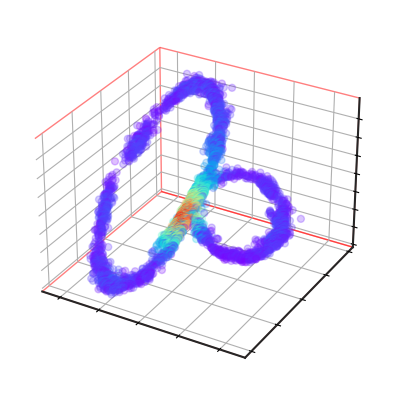}
    \includegraphics[width = 0.195\linewidth]{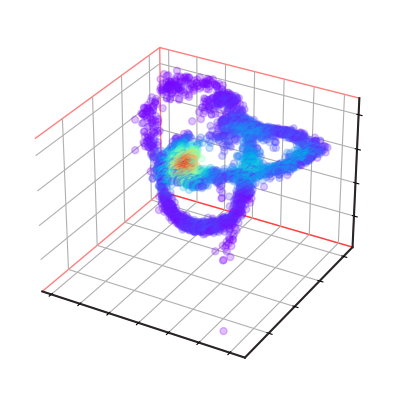}
    \caption{Qualitative visualization of the samples generated by \textbf{CEF} (Middle Row) and its VQ-counterpart (Bottom Row) trained on Toy 3D data distributions (Top Row). CEF consists of a 2-dimensional RealNVP flow post composed with a conformal embedding that raises it to 3D. }
\end{figure*}

\begin{figure*}[h!]
    \centering
    \begin{tabular}{cc}     \subfloat[InterlockedCircles]{\includegraphics[width = 0.2\linewidth]{figures/datasets/realdata_InterlockedCIRCLES.png}}
    \subfloat[Bent-Lissajous]{\includegraphics[width = 0.2\linewidth]{figures/datasets/realdata_BentLISSAJOUS.png}}
    \subfloat[Non-Knotted]{\includegraphics[width = 0.2\linewidth]{figures/datasets/realdata_NonKNOTTED.png}}
    \subfloat[Disjoint-Circles]{\includegraphics[width = 0.2\linewidth]{figures/datasets/realdata_DisjointCIRCLES.png}}
    \subfloat[Star]{\includegraphics[width = 0.2\linewidth]{figures/datasets/realdata_STAR.png}}
    \end{tabular}\\
    \includegraphics[width = 0.195\linewidth]{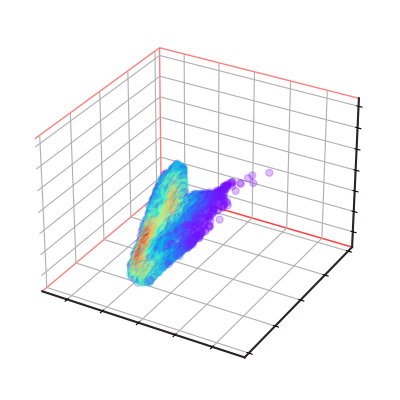}
    \includegraphics[width = 0.195\linewidth]{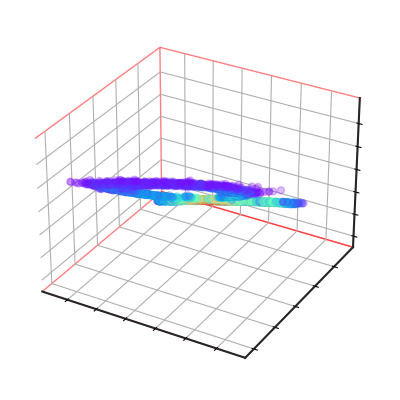}
    \includegraphics[width = 0.195\linewidth]{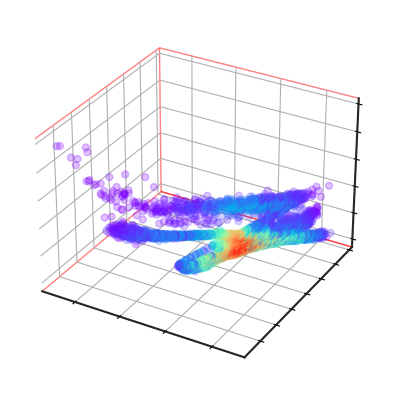}
    \includegraphics[width = 0.195\linewidth]{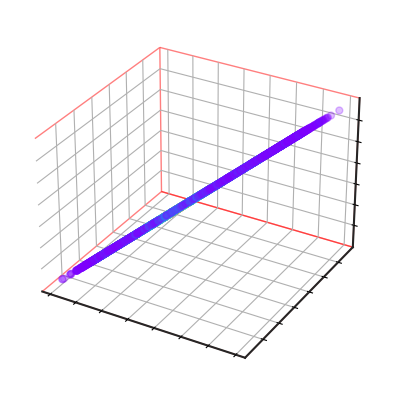}
    \includegraphics[width = 0.195\linewidth]{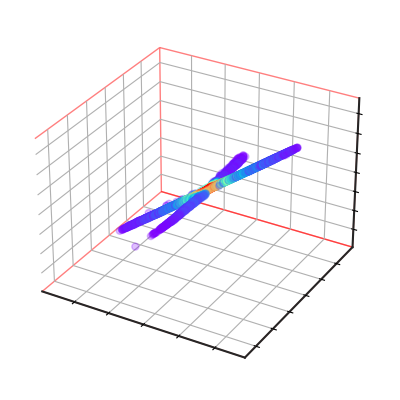} \\ 
    \includegraphics[width = 0.195\linewidth]{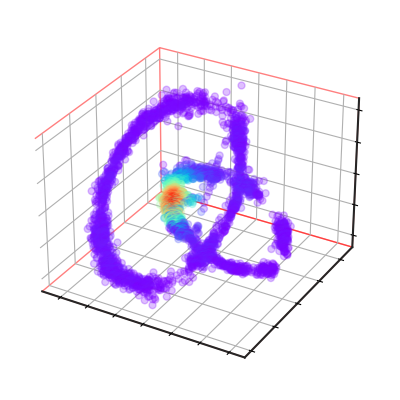} 
    \includegraphics[width = 0.195\linewidth]{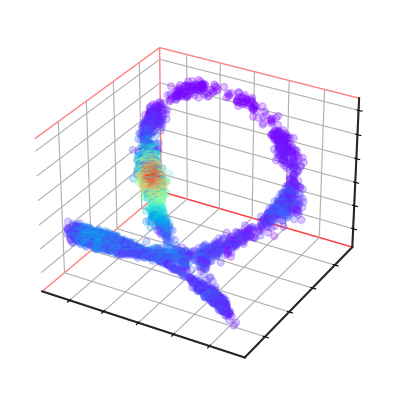}
    \includegraphics[width = 0.195\linewidth]{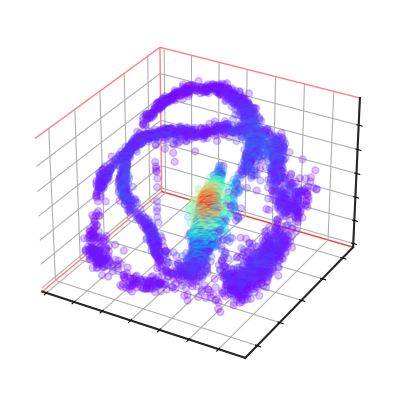}
    \includegraphics[width = 0.195\linewidth]{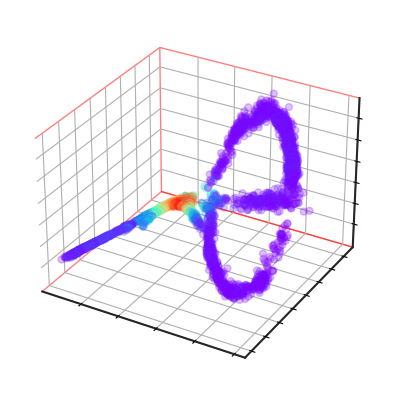}
    \includegraphics[width = 0.195\linewidth]{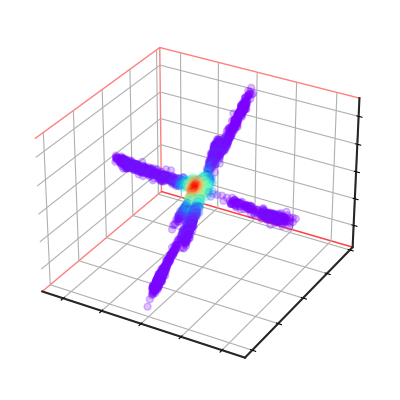}
    \caption{Qualitative visualization of the samples generated by  \textbf{CEF} (Middle Row) and its VQ-counterpart (Bottom Row) trained on Toy 3D data distributions (Top Row). CEF consists of a 2-dimensional RealNVP flow post composed with a conformal embedding that raises it to 3D. }
    \label{qual-sample-end}
\end{figure*}

\bibliographystyle{unsrt}  
\bibliography{references}